\pdfoutput=1

\RequirePackage[l2tabu, orthodox]{nag}
\documentclass[12pt]{article}
\PassOptionsToPackage{numbers, compress, sort}{natbib}
\usepackage[T1]{fontenc}
\usepackage{microtype}

\usepackage{amsmath}
\usepackage{amssymb}
\usepackage{amsthm}
\usepackage{mathtools}
\usepackage{thmtools}
\usepackage{mathrsfs}

\usepackage{libertine}      

\usepackage{titlesec}

\usepackage{hyperref}

\titleformat*{\section}{\Large\bfseries}

\usepackage[usenames,dvipsnames]{xcolor}
\definecolor{shadecolor}{gray}{0.9}

\usepackage[english]{babel}
\usepackage[parfill]{parskip}
\usepackage{afterpage}
\usepackage{framed}
\usepackage{nicefrac}

{\endMakeFramed}

\newcounter{parcount}

\usepackage{graphicx}
\usepackage[labelfont=bf]{caption}
\usepackage[format=hang]{subcaption}
\usepackage{wrapfig}

\usepackage{booktabs}
\usepackage{multirow}
\usepackage{longtable}

\usepackage[sort]{natbib}
\usepackage{bibunits}

\usepackage[algoruled]{algorithm2e}
\usepackage{listings}
\usepackage{fancyvrb}
\fvset{fontsize=\normalsize}

\usepackage{hyperref}
\hypersetup{colorlinks,linktoc=all}
\hypersetup{citecolor=Violet}
\hypersetup{linkcolor=black}
\hypersetup{urlcolor=MidnightBlue}

\usepackage[nameinlink]{cleveref}

\usepackage[acronym,nowarn]{glossaries}

\usepackage{listings}
\usepackage{lstbayes}
\lstdefinestyle{mystyle}{
    commentstyle=\color{OliveGreen},
    numberstyle=\tiny\color{black!60},
    stringstyle=\color{BrickRed},
    basicstyle=\ttfamily\scriptsize,
    breakatwhitespace=false,
    breaklines=true,
    captionpos=b,
    keepspaces=true,
    numbers=none,
    numbersep=5pt,
    showspaces=false,
    showstringspaces=false,
    showtabs=false,
    tabsize=2
}
\usepackage{enumitem}

\crefname{lemma}{lemma}{lemmas}
\Crefname{lemma}{Lemma}{Lemmas}
\crefname{theorem}{theorem}{theorems}
\Crefname{theorem}{Theorem}{Theorems}
\crefname{prop}{proposition}{propositions}
\Crefname{prop}{Proposition}{Propositions}

\newtheorem{thm}{Theorem} % reset theorem numbering for each chapter
\newtheorem{prop}[thm]{Proposition}
\newtheorem{lemma}[thm]{Lemma}

\renewcommand{\mid}{~\vert~}

\setacronymstyle{long-sc-short}
\newacronym{KL}{kl}{Kullback-Leibler}

\newacronym{POPELBO}{pop-elbo}{\emph{population evidence lower bound}}
\newacronym{PROELBO}{pro-elbo}{\emph{profile evidence lower bound}}

\newacronym{SVI}{svi}{stochastic variational inference}
\newacronym{VI}{vi}{variational inference}

\newacronym{ADVI}{advi}{automatic differentiation variational inference}

\newacronym{LDA}{lda}{latent Dirichlet allocation}

\newacronym{SMC}{smc}{Sequential Monte Carlo}
\newacronym{VB}{vb}{variational Bayes}

\newacronym{TDVI}{tdvi}{transdimensional variational inference}
\newacronym{STDVI}{stdvi}{sequential transdimensional variational inference}
\newacronym{MCMC}{mcmc}{Markov chain Monte Carlo}
\newacronym{RJMCMC}{rjmcmc}{reversible jump Markov chain Monte Carlo}
\newacronym{TDMCMC}{tdmcmc}{transdimensional Markov chain Monte Carlo}

\newacronym{SLDS}{slds}{switching linear dynamical system}
\newacronym{HDP-SLDS}{hdp-slds}{hierarchical Dirichlet process switching linear dynamical system}

\newacronym{MLM}{mlm}{masked language model}
\newacronym{CBOW}{cbow}{continuous bag of words}
\newacronym{MoE}{moe}{mixture-of-experts}
\newacronym{SGNS}{sgns}{skip-gram with negative sampling}
\newacronym{LLM}{llm}{large language model}
\newacronym{AWE}{awe}{attention word embedding}
\newacronym{OOD}{ood}{out-of-distribution}

\usepackage[margin=1in]{geometry}
\usepackage{tcolorbox}
\usepackage{adjustbox}

\usepackage[defaultlines=3,all]{nowidow}

\usepackage{libertine}  
\usepackage{dsfont}

\usepackage{caption}
\usepackage{subcaption}
\usepackage{wrapfig}

\DeclareCaptionFormat{empty}{}

\everypar{\looseness=-1}

\title{From Unstructured Data to In-Context Learning: Exploring What Tasks Can Be Learned and When}

\author{
  Kevin Christian Wibisono\\
        University of Michigan, Statistics\\
  kwib@umich.edu\\
  \and
  Yixin Wang\\
          University of Michigan, Statistics\\
  yixinw@umich.edu\\
  }

\date{\today}

\begin{document}
\maketitle

%\begin{bibunit}[abbrvnat]
% !TEX root = attention.tex

\begin{abstract}
Large language models (LLMs) like transformers demonstrate impressive
in-context learning (ICL) capabilities, allowing them to make
predictions for new tasks based on prompt exemplars without parameter
updates. While existing ICL theories often assume structured training
data resembling ICL tasks (e.g., x-y pairs for linear regression),
LLMs are typically trained unsupervised on unstructured text, such as
web content, which lacks clear parallels to tasks like word analogy.
To address this gap, we examine what enables ICL in models trained on
unstructured data, focusing on critical sequence model requirements
and training data structure. We find that many ICL capabilities can
emerge simply from co-occurrence of semantically related word pairs in
unstructured data; word analogy completion, for example, can provably
arise purely through co-occurrence modeling, using classical language
models like continuous bag of words (CBOW), without needing positional
information or attention mechanisms. However, positional information
becomes crucial for logic reasoning tasks requiring generalization to
unseen tokens. Finally, we identify two cases where ICL fails: one in
logic reasoning tasks that require generalizing to new, unseen
patterns, and another in analogy completion where relevant word pairs
appear only in fixed training positions. These findings suggest that
LLMs' ICL abilities depend heavily on the structural elements within
their training data.\footnote{Software that replicates the empirical
studies can be found at
\url{https://github.com/yixinw-lab/icl-unstructured}. Details on
implementation, experiments and data sets are provided in Appendix
\ref{app:exp-details}.}
\end{abstract}

positional embeddings, transformers

\section{Introduction}
\label{sec:intro}
Large language models (LLMs) such as transformers demonstrate remarkable in-context learning (ICL) abilities~\citep{brown2020language}: without any parameter updates, they can recognize tasks and generate predictions from prompt examples.  For instance, given the prompt \textit{dog anjing, cat kucing, lion singa, elephant}, a well-trained LLM should detect the English-to-Indonesian pattern in the prompt and predict \textit{gajah}---the Indonesian translation for \textit{elephant}---as the most likely next token. The ICL capabilities of LLMs are surprising for two main reasons. First, these models are trained in an unsupervised manner on unstructured natural language data through next-token prediction, without any loss function specifically designed for ICL. Second, the training data for LLMs likely lacks sequences resembling typical ICL prompts, i.e., of the form $c_1d_1 \cdots c_Kd_K$, where $(c_k,d_k)$'s represent word pairs with specific semantic relationships.

Many efforts have sought to understand ICL from theoretical and empirical perspectives, e.g., gradient descent in regression and Bayesian inference. While insightful, these analyses often rely on structured training data that mirrors ICL tasks. For instance, they train on sequences of x-y pairs from various linear regression tasks and test on similar data. In practice, however, LLMs are trained in an unsupervised manner on unstructured text data, such as web content, which bears little resemblance to typical ICL tasks like word analogy. Consequently, these analyses may only partially capture the complexities of ICL.

\textbf{This work.} We investigate common ICL tasks to identify what tasks can be learned in context by a model trained on unstructured data. Specifically, we examine essential components of sequence modeling that enable in-context learning, along with requirements on the unstructured training data.

The first set of (theoretical and empirical) results focuses on ICL for word analogy completion using frequently co-occurring tokens~\citep{brown2020language,todd2024function}. This task involves identifying relationships between word pairs, such as \textit{(country)-(capital)} and \textit{(English word)-(Indonesian translation)}, then applying the same relationship to complete a sequence. For this task (see left of Figure \ref{fig:illustration}), we explore cases where training sentences contain one or two types of word pairs with distinct semantic relationships. We prove that, in most cases, ICL can arise by simply modeling word co-occurrence using classical (pre-transformer) language models like continuous bag of words (CBOW) \citep{mikolov2013efficient}, without needing positional information or attention mechanisms. 

% We further validate this result with prompting and synthetic data experiments.

The second set of results involves ICL for logic reasoning tasks that require recognizing patterns that do not commonly co-occur in a sentence, such as \textit{(word)-(first letter)}~\citep{xu2024large,chen2024parallel}. For this task (see middle of Figure \ref{fig:illustration}), we investigate scenarios where training sentences contain one or two distinct patterns, as well as a more realistic scenario where \textcolor{violet}{nuisance tokens} are present. We prove that \textit{positional information and blocked nuisance structure} (e.g., \textit{\textcolor{violet}{pqrs}} in Figure \ref{fig:illustration}) \textit{are crucial for the success of ICL in these tasks}. This finding aligns with \citeauthor{chen2024parallel}'s [\citeyear{chen2024parallel}] observation that parallel structures in pre-training data support ICL. We also find that learned positional embeddings generally perform better, except in scenarios where the nuisance tokens are not clustered in blocks.

\begin{figure}
  \centering
  \includegraphics[height=4.8cm]{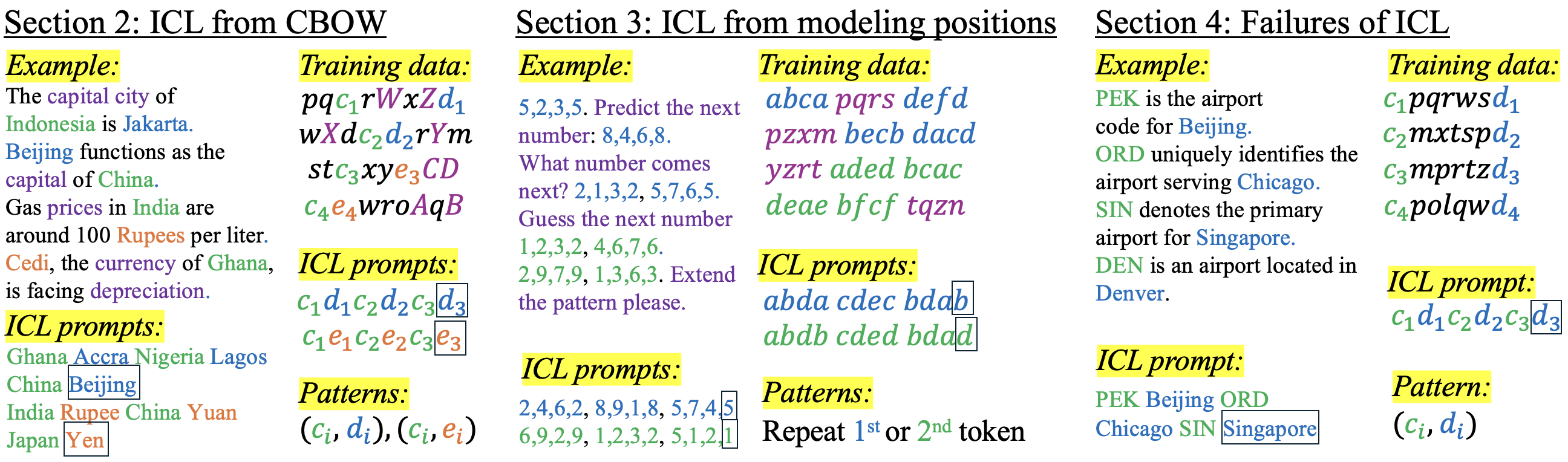}
  %\hspace{0.5cm}
   \caption{This paper identifies essential components for in-context learning (ICL) from pre-training on unstructured natural language data. Left sub-panels, right sub-panels, and boxed letters denote NLP examples, our abstractions, and expected outputs, respectively. Section \ref{sec:co-occ} shows that ICL for word analogy tasks can arise via modeling co-occurrence information using classical language models like continuous bag of words (CBOW) (\textcolor{violet}{violet} represents relationship-specific nuisance tokens). Section \ref{sec:icl-pos} establishes the necessity of modeling positional information and blocked nuisance structure for ICL tasks, enabling pattern recognition and generalization to novel tokens (\textcolor{violet}{violet} represents nuisance tokens). Section \ref{sec:failed} presents scenarios where ICL fails, providing theoretical explanations that underscore the critical role of training data structure in enabling ICL in language models.}   
  \label{fig:illustration}
\end{figure}

Finally, we present two scenarios where ICL fails regardless of model architectures (see right of Figure \ref{fig:illustration}). In the first scenario (left example), we consider a logic reasoning task that involves identifying and completing meta-patterns within sequences. Here, each training sequence repeats the pattern established by its starting tokens; the ICL task sequence then requires the model to recognize this meta-pattern of repetition and generalize it to a novel, unseen starting pattern. In the second scenario (right example), we examine a word analogy completion task in which relevant word pairs appear in the unstructured training sentences but are restricted to specific fixed positions. These findings, along with their empirical and theoretical explanations, underscore that \textit{LLMs require specific structures in the pre-training data to exhibit ICL ability}.

\textbf{Summary of contributions.} We (1) theoretically and empirically show that ICL for word analogy tasks with semantically related word pairs can arise from modeling \textit{co-occurrence patterns} using CBOW, (2) prove that, to recognize token patterns and generalize them to novel tokens, ICL requires \textit{modeling positional information} and \textit{blocked nuisance structure}, and (3) present scenarios where ICL fails, highlighting the crucial role of \textit{training data structure} for ICL.

\textbf{Related work.} Below, we highlight some of these studies and explain how our research aligns with, yet differs from, these approaches. We include a detailed discussion of related work in \Cref{app:rel-work}.
Numerous studies have connected ICL to classical methods, including gradient descent \citep{akyurek2022learning, oswald2023transformers, dai2023why, zhang2024trained, ahn2024transformers}, Bayesian inference \citep{wang2023large,zhang2023what,chiang2024understanding}, and Newton’s method~\citep{fu2023transformers}. In contrast, \textit{our work links ICL to the continuous bag of words (CBOW) model}, showing that ICL for word analogy tasks can be achieved by learning co-occurrence patterns.
Several studies have examined the pre-training aspects of ICL, such as data distribution \citep{min2022rethinking, chan2022data, kossen2023in} and task diversity \citep{raventos2023pretraining, yadlowsky2023pretraining}. By comparison, \textit{our work emphasizes the importance of co-occurrence, positional information, and training data structure} for ICL to arise.
Other research has explored ICL in specific data-generating processes, such as discrete functions \citep{bhattamishra2023understanding} and autoregressive processes \citep{sander2024how}. In contrast, \textit{our work centers on data characterized by semantically related word pairs and repeating token patterns}.

%\vspace{-4mm}
\section{In-context learning can arise by modeling co-occurrence via CBOW}
\label{sec:co-occ}
%\vspace{-2mm}
In this section, we focus on in-context learning (ICL) for word analogy tasks involving word pairs that frequently co-occur in training sentences; see \Cref{fig:illustration} (left). To motivate the discussion, we present two experiments using the LLaMA 2 model \citep{touvron2023llama} involving countries (or US states) and their capital cities (see Appendix \ref{app:exp-details} for data sources). The prompts follow the format $c_1 d_1, c_2 d_2, \cdots, c_6 d_6, c_7$, where $c_i$ is a country (or US state) and $d_i$ is its capital city. In this scenario, we consider ICL successful if the model outputs $d_7$---the capital city of $c_7$---as the most likely token.

\textbf{Experiment 1.} We consider all 160 countries with a population exceeding one million in 2022. Among these countries, 31 have capital cities that are not their most populous cities, denoted by \textit{type A}. The remaining 129 countries fall under \textit{type B}. Each ICL prompt includes three type A countries among $c_1, \cdots, c_6$ to emphasize that the desired relationship is \textit{(country)-(capital)} rather than \textit{(country)-(largest city)}. Subsequently, we randomly generate 1,000 prompts, with 500 having a $c_7$ being a type A country and 500 having a $c_7$ being a type B country.  The model's ICL accuracy is $0.58$ for type A and $0.96$ for type B.

\textbf{Experiment 2.} We consider all 50 US states, among which 33 are of \textit{type A} and 17 are of \textit{type B}, similarly defined. Following the setup in Experiment 1, we generate prompts for these states. The ICL accuracy is $0.69$ for type A and $0.84$ for type B.

In both experiments, LLaMA 2 performs better on type B prompts (i.e., the capital city as the largest city). Since larger cities tend to appear more frequently as compared to smaller ones in the model's pre-training data, this naturally raises the question: \textit{Can/does ICL with frequently co-occuring word pairs arise purely from modeling co-occurrence patterns?}

\textbf{ICL via classical non-transformer-based language models.} We prove that, for word analogy tasks with frequently co-occurring word pairs, ICL can be achieved by modeling token co-occurrence---without needing positional encoding or attention mechanisms---using classical, non-transformer language models such as the continuous bag of words (CBOW) model \citep{mikolov2013efficient}. (It \textit{does not} imply that ICL in transformer-based models arises through learning co-occurrence patterns.) We utilize CBOW variant where each center word is modeled conditional on all other words in a sentence, not just neighboring words. Specifically, each word $w$ has center and context embeddings $u_w$ and $v_w$ of the same dimension. Given a sentence $x_1 x_2 \cdots x_I$, the $i$-th word ($x_i$) is distributed as $p(x_i = k \mid x_{-i}) \propto \exp( (u_k^\top  \sum_{j \neq i} v_{x_j})/(I-1))$, with $u_w$'s and $v_w$'s learned by minimizing cross-entropy losses across all positions.

\textbf{Roadmap of Section \ref{sec:co-occ}.} In Section \ref{sec:icl-single}, we begin by considering a simple ICL task of the form $c_{i_1} d_{i_1} \cdots c_{i_{\ell}} d_{i_{\ell}} c_{i_{\ell+1}}$, where $(c_i, d_i)$ represents a frequently co-occuring word pair  (e.g., a country and its capital city) and $i_1, i_2, \cdots, i_{\ell+1}$ are all distinct. The focus is to investigate whether a trained CBOW model can correctly output $d_{i_\ell}$.
We also explore two other scenarios: ICL tasks of the form $c_{i_1} d_{i_1} \cdots c_{i_{\ell}} d_{i_{\ell}} c_{i_{\ell+1}}$ and $c_{i_1} e_{i_1} \cdots c_{i_{\ell}} e_{i_{\ell}} c_{i_{\ell+1}}$ in Section \ref{sec:icl-connected} (two connected word relationships), as well as $c_{i_1} d_{i_1} \cdots c_{i_{\ell-1}} d_{i_{\ell-1}} c_{i_\ell}$ and $e_{i_1} f_{i_1} \cdots e_{i_{\ell}} f_{i_{\ell}} e_{i_{\ell+1}}$ (two disjoint word relationships) in Section \ref{sec:icl-disconnected}. Section \ref{sec:synth-exp} concludes with synthetic experiments supporting the theory.

%\vspace{-2mm}
\subsection{In-context learning on single-relationship word analogy tasks}
\label{sec:icl-single}

We investigate ICL in single-relationship word analogy tasks, where the training data contains only one type of relationship between frequently co-occurring word pairs. This task takes the form of $c_{i_1} d_{i_1} \cdots c_{i_{\ell}} d_{i_{\ell}} c_{i_{\ell+1}}$, where $(c_i, d_i)$ pair represents a frequent co-occurrence, such as a country and its capital city. The vocabulary consists of $c_{1:K}, d_{1:K}, r_{1:L}$, where $r_i's$ represent other words (e.g., stop words). We first introduce Theorem \ref{prop:icl-cbow}, which states that ICL can arise if each sentence consists of exactly one $(c_i, d_i)$ pair, as long as the number of in-context examples ($\ell$) is not too large. To simplify calculations, we replace the cross-entropy loss with squared loss by removing the softmax activation and comparing outputs against the one-hot encoding of target words. The proof is in Appendix \ref{app:proof-1}. 

\begin{thm}[ICL on single-relationship word analogy tasks]
\label{prop:icl-cbow}
    Let $K, L \geq S \geq 3$. Suppose each training sentence of length $S$ is generated by selecting one $(c_i, d_i)$ pair and $S - 2$ distinct $r_i$'s uniformly at random. We train a CBOW model with the squared loss and a sufficiently large embedding dimension on these sentences. Given a prompt $c_{i_1} d_{i_1} \cdots c_{i_{\ell}} d_{i_{\ell}} c_{i_{\ell+1}}$ with distinct $i_k$'s, the model correctly predicts $d_{i_{\ell+1}}$ if and only if $$2\ell + 1 < \frac{KL(S-1)^3}{(K+L)(S-2)^2(S-1)+K(S-2)(S-1)^2-2(S-2)^4}.$$
\end{thm}
%\vspace{-2mm}
As an example, when each training sentence contains exactly one \textit{country-capital} pair (i.e., $(c_i, d_i)$), Theorem~\ref{prop:icl-cbow} says that a trained CBOW model will correctly predict $d_{i_{\ell+1}}$ (i.e., the capital city of $c_{i_{\ell+1}}$) given an ICL prompt of the form $c_{i_1} d_{i_1} \cdots c_{i_{\ell}} d_{i_{\ell}} c_{i_{\ell+1}}$, provided that the prompt length is not too large. Intuitively, this behavior is due to the presence of $c_{i_{\ell+1}}$ in the ICL prompt, leading the model to correctly predict $d_{i_{\ell+1}}$ given the frequent occurrences of the pair $(c_{i_{\ell+1}}, d_{i_{\ell+1}})$ in the training data. However, when the prompt length is too large, the model will instead predict one of the $r_i$'s (see Theorem \ref{prop:icl-cbow}'s proof in Appendix \ref{app:proof-1} for more details). Moreover, if we let $L \rightarrow \infty$ and fix $K$ and $S$, the condition in Theorem \ref{prop:icl-cbow} becomes
$2\ell + 1 < K (S-1)^2 / (S-2)^2$. This inequality trivially holds if the prompt length is set to be $S-1$ to match the length of the training sentences. 

Furthermore, it is possible to adapt the proof of Theorem \ref{prop:icl-cbow} to handle the case when each sentence comprises exactly two (not one) different $(c_i, d_i)$ pairs. In this case, letting $L \rightarrow \infty$ and fixing $K$ and $S$, the model correctly predicts $d_{i_{\ell+1}}$ given the same ICL prompt if and only if $2\ell + 1 < \frac{K(K-2)(S-1)^2}{(K-2)(S-2)(S-4)-K}$. This upper bound is strictly larger than $K (S-1)^2 / (S-2)^2$: when each sentence contains exactly two $(c_i, d_i)$ pairs, ICL under the squared loss holds for longer prompts. 

\textbf{Experiments.} To empirically verify Theorem \ref{prop:icl-cbow} and its generalizations, we conduct experiments using the cross-entropy loss with $S = 8$, $K = 10$, $L = 20$, and $\ell = 3$. We explore multiple $(p_0, p_1, p_2)$ values, where $p_k$ denotes the probability of having exactly $k$ pairs of $(c_i, d_i)$ in the sentence. For each $(p_0, p_1, p_2)$ triple, we also introduce a more realistic setting where $c_i$ and $d_i$ do not always appear together by considering its \textit{corrupted} version. In this setup, each $(c_i, d_i)$ pair has a 25\% chance of being replaced with $(c_i, r_j)$ and a 25\% chance of being replaced with $(d_i, r_j)$ for some $j \in [L]$. More details are provided in Appendix \ref{app:exp-details}.

\begin{table}[t]
    \centering
    \caption{ICL on different single-relationship word analogy tasks, averaged over 10 repetitions, demonstrates stable, good performance across embedding dimensions ($d_E$), as Theorem~\ref{prop:icl-cbow} suggests. The corrupted setting also demonstrates excellent ICL ability under certain scenarios.}
    \begin{tabular}{c c c c c}
    \toprule
    \multicolumn{1}{c}{} & \multicolumn{2}{c}{Clean} & \multicolumn{2}{c}{Corrupted} \\
    \cmidrule(lr){2-3} \cmidrule(lr){4-5}
    $(p_0, p_1, p_2)$ & $d_E$ = 10 & $d_E$ = 100 & $d_E$ = 10 & $d_E$ = 100 \\
    \midrule
    $(0, 1, 0)$ & 0 & 0 & 0 & 0 \\
    $(0, 0, 1)$ & 0 & 0 & 0 & 0 \\
    $(1/2, 1/2, 0)$ & 1 & 0.99 & 0 & 0 \\
    $(1/2, 0, 1/2)$ & 1 & 1 & 1 & 1 \\
    $(0, 1/2, 1/2)$ & 1 & 1 & 0 & 0.01 \\
    $(1/3, 1/3, 1/3)$ & 1 & 1 & 1 & 1 \\
    \bottomrule
    \end{tabular}
    \label{tab:one-rel}
\end{table}

\textbf{Results.} Table \ref{tab:one-rel} displays the average accuracy for each scenario, calculated over 10 repetitions. Notably, when $(p_0,p_1,p_2)$ is $(0,1,0)$ or $(0,0,1)$, ICL under the cross-entropy loss achieves zero accuracy, in contrast to perfect accuracy with the squared loss as shown in Theorem \ref{prop:icl-cbow}. We believe this difference in accuracy is an artifact of the loss functions used, although its relevance is limited by the fact that it is unlikely for every sentence to contain at least one $(c_i, d_i)$ pair, in reality. On the other hand, perfect ICL performance is observed in other settings (e.g., when the training sentences contain either zero, one, or two $(c_i, d_i)$ pairs) in both the clean and corrupted scenarios. For an in-depth comparison of ICL performance using both the squared and cross-entropy loss across various numbers of demonstration examples, see Appendix \ref{app:comp-loss-ex}.

\subsection{In-context learning on dual-connected-relationship word analogy tasks}
\label{sec:icl-connected}
Building on the scenario that contains only a single type of relationship between frequently co-occurring word pairs, namely $(c_i, d_i)$, we now explore ICL on dual-connected-relationship word analogy tasks. Here, some words frequently co-occur with two different types of words in the training data, represented by the relationships $(c_i, d_i)$ and $(c_i, e_i)$. For instance, $c_i$ might represent a country, $d_i$ its capital city, and $e_i$ its currency. The vocabulary is comprised of $c_{1:K}, d_{1:K}, e_{1:K}, r_{1:L}$, where $r_i$'s represent other words. Moreover, the corresponding ICL tasks take the form $c_{i_1} d_{i_1} \cdots c_{i_{\ell}} d_{i_{\ell}} c_{i_{\ell+1}}$ and $c_{i_1} e_{i_1} \cdots c_{i_{\ell}} e_{i_{\ell}} c_{i_{\ell+1}}$, where the model is expected to output $d_{i_{\ell+1}}$ and $e_{i_{\ell+1}}$, respectively. These can be regarded as \textit{task selection} since the model should use the in-context examples to infer the tasks. We present Theorem \ref{prop:icl-cbow-2}, stating that a trained CBOW model can perform task selection if each sentence contains exactly two distinct $(c_i, d_i)$ or two distinct $(c_i, e_i)$ pairs with uniform probability.{\footnote{{We can also theoretically show that ICL works (up to a certain number of training examples) in this scenario, but the calculations are extremely tedious. Therefore, we only present empirical evidence in Table \ref{tab:two-con}.}}}

\begin{thm}[Task selection in CBOW]
    \label{prop:icl-cbow-2}
    Let $K, L \geq 2$ and $S \geq 5$. Suppose each training sentence of length $S$ is generated by selecting two distinct $(c_i, d_i)$ pairs or $(c_i, e_i)$ pairs, and $S - 4$ distinct $r_i$'s uniformly at random. We train a CBOW model with the squared loss and a large enough embedding dimension. Given a prompt $c_{i_1} d_{i_1} \cdots c_{i_{\ell}} d_{i_{\ell}} c_{i_{\ell+1}}$ ($c_{i_1} e_{i_1} \cdots c_{i_{\ell}} e_{i_{\ell}} c_{i_{\ell+1}}$) with distinct $i_k$'s, the model is more likely to predict $d_{i_{\ell+1}}$ ($e_{i_{\ell+1}}$) than $e_{i_{\ell+1}}$ ($d_{i_{\ell+1}}$). (The proof is in Appendix \ref{app:proof-thm-2}.)
\end{thm}
%\vspace{-2mm}
According to Theorem \ref{prop:icl-cbow-2}, when each training sentence includes two $(c_i, d_i)$ pairs or two $(c_i, e_i)$ pairs, a trained CBOW model is capable of performing task selection. To intuitively understand this result, consider the ICL prompt of the first type, i.e., $c_{i_1} d_{i_1} \cdots c_{i_{\ell}} d_{i_{\ell}} c_{i_{\ell+1}}$. Here, the output is more likely to be $d_{i_{\ell+1}}$ than $e_{i_{\ell+1}}$ since $d_{i_{\ell+1}}$ co-occurs with the other $d_{i_j}$'s in the training data (and $e_{i_{\ell+1}}$ does not). Note that in Theorem~\ref{prop:icl-cbow-2}, we unrealistically require each sentence to contain either two distinct $(c_i, d_i)$ pairs or $(c_i, e_i)$ pairs. However, this condition is not necessary as we empirically show next.

\textbf{Experiments.} We use the cross-entropy loss with $S = 8$, $K = 10$, $L = 60$, and $\ell = 3$. Each training sentence is equally likely to be a \textit{cd} sentence (i.e., containing $(c_i, d_i)$ pairs) or a \textit{ce} sentence (i.e., containing $(c_i, e_i)$ pairs), but not both. We explore multiple $(p_0, p_1, p_2)$'s, where $p_k$ is the probability of having exactly $k$ pairs of $(c_i, d_i)$ for a \textit{cd} sentence, or $k$ pairs of $(c_i, e_i)$ for a \textit{ce} sentence. Additionally, we introduce three different scenarios: \textit{balanced}, where all $L$ random words are equally likely to occur in both \textit{cd} and \textit{ce} sentences; \textit{imbalanced}, where $L/3$ words are more likely to occur in \textit{cd} (\textit{ce}) sentences; and \textit{extreme}, where $L/3$ of the words can only occur in \textit{cd}  (\textit{ce}) sentences. More details are provided in Appendix \ref{app:exp-details}.

\begin{table}[t]
  \centering
  \setlength{\tabcolsep}{5pt}
    \caption{ICL on dual-\emph{connected}-relationship tasks, averaged over 10 repetitions, achieves perfect accuracy when $(p_0, p_1, p_2) \in \{(1/2, 0/ 1,2), (0, 1/2, 1/2), (1/3, 1/3, 1/3)\}$ regardless of architectures and embedding dimensions ($d_E$), as Theorem \ref{prop:icl-cbow-2} suggests. When $(p_0, p_1, p_2) = (1/2, 1/2, 0)$, ICL performs better under imbalanced or extreme scenarios and with larger $d_E$.}
    %\vspace{2mm}
    \begin{tabular}{c c c c c c c}
    \toprule
    \multicolumn{1}{c}{} & \multicolumn{2}{c}{Balanced} & \multicolumn{2}{c}{Imbalanced} & \multicolumn{2}{c}{Extreme}\\
    \cmidrule(lr){2-3} \cmidrule(lr){4-5} \cmidrule(lr){6-7}
    $(p_0, p_1, p_2)$ & $d_E$ = 10 & $d_E$ = 100 & $d_E$ = 10 & $d_E$ = 100 & $d_E$ = 10 & $d_E$ = 100\\
    \midrule
    $(0, 1, 0)$ & (0, 0) & (0, 0) & (0, 0) & (0, 0) & (0, 0) & (0, 0) \\
    $(0, 0, 1)$ & (0, 0) & (0, 0) & (0, 0) & (0, 0) & (0.07, 0.10) & (0, 0) \\
    $(1/2, 1/2, 0)$ & (0.53, 0.47) & (0.51, 0.50) & (0.69, 0.68) & (1, 1) & (0.94, 0.93) & (1, 1) \\
    $(1/2, 0, 1/2)$ & (1, 1) & (1, 1) & (1, 1) & (1, 1) & (1, 1) & (1, 1)\\
    $(0, 1/2, 1/2)$ & (1, 1) & (1, 1) & (1, 1) & (1, 1) & (1, 1) & (1, 1) \\
    $(1/3, 1/3, 1/3)$ & (1, 1) & (1, 1) & (1, 1) & (1, 1) & (1, 1) & (1, 1) \\
    \bottomrule
    \end{tabular}
    \label{tab:two-con}
    \\[-4mm] % Adjust negative space for reducing the gap
\end{table}

\textbf{Results.} Table \ref{tab:two-con} shows the accuracies of both tasks for each scenario, averaged over 10 repetitions. We observe a perfect accuracy when $(p_0, p_1, p_2) \in \{(1/2, 0/ 1,2), (0, 1/2, 1/2), (1/3, 1/3, 1/3)\}$ across all embedding dimensions and scenario types. The near-zero accuracy when $(p_0,p_1,p_2)$ or $(0,1,0)$ or $(0,0,1)$ is again an artifact of the cross-entropy loss, as discussed in Section \ref{sec:icl-single}.

Interestingly, ICL works in the imbalanced and extreme scenarios when $(p_0, p_1, p_2) = (1/2, 1/2, 0)$, where sentences do not contain more than one $(c_i, d_i)$ or $(c_i, e_i)$ pair. To see this, consider the balanced scenario where each $r_i$ is equally probable to appear in both types of sentences. Given a prompt of the form $c_{i_1} d_{i_1} \cdots c_{i_{\ell}} d_{i_{\ell}} c_{i_{\ell+1}}$, it is easy to see that the model should output $d_{i_{\ell+1}}$ or $e_{i_{\ell+1}}$ with equal probability. On the other hand, in the imbalanced and extreme scenarios, the information from the $r_i$'s can allow for task selection, thus contributing to the success of ICL.
%\vspace{-2mm}

\subsection{In-context learning on dual-disjoint-relationship tasks}
%\vspace{-2mm}
\label{sec:icl-disconnected}

We next replicate the experiments in Section \ref{sec:icl-connected}, but with disjoint word pair relationships of two distinct types with no overlapping tokens, i.e., $(c_i, d_i)$ and $(e_i, f_i)$. For example, $(c_i, d_i)$ represents a country and its capital city, and $(e_i, f_i)$ represents a company and its CEO. Our vocabulary consists of $c_{1:K}, d_{1:K}, e_{1:K}, f_{1:K}, r_{1:L}$, where $r_i$'s represent other words; see Appendix \ref{app:exp-details} for details.

\textbf{Results.} Table \ref{tab:two-discon} presents the accuracies of the ICL tasks $c_{i_1} d_{i_1} \cdots c_{i_{\ell}} d_{i_{\ell}} c_{i_{\ell+1}}$ and $e_{i_1} f_{i_1} \cdots e_{i_{\ell}} f_{i_{\ell}} e_{i_{\ell+1}}$ for each scenario, averaged over 10 repetitions. Similar to the connected setting in Section \ref{sec:icl-connected}, we observe a perfect accuracy when $(p_0, p_1, p_2) \in \{(1/2, 0/ 1,2), (0, 1/2, 1/2), (1/3, 1/3, 1/3)\}$ across all embedding dimensions and scenario types. However, when $(p_0, p_1, p_2) = (1/2, 1/2, 0)$, ICL already works well in the balanced scenario. Intuitively, this is because the two relationships are disjoint, thus making task selection easier. 

In addition, we consider a \textit{contaminated} version of the training data where \textit{cd} (\textit{ef}) sentences can contain some $e_i$'s and $f_i$'s ($c_i$'s and $d_i$'s). We also obtain a perfect accuracy when $(p_0, p_1, p_2)$ is in $\{(1/2, 0/ 1,2), (0, 1/2, 1/2), (1/3, 1/3, 1/3)\}$ across all embedding dimensions and scenario types.

% \subsection{Experiments on countries, US states, and their capital cities}
% \label{sec:co-occ-exp}

% We perform two experiments involving countries and their capital cities, as well as US states and their capital cities. Our prompts follow the format $c_1 d_1, c_2 d_2, \cdots, c_6 d_6, c_7$, where $c_i$ is a country or US state and $d_i$ is its capital city. Using the LLaMA 2 model \citep{touvron2023llama}, we compare the prediction for each prompt with its corresponding $d_7$. The experimental results support the theory. 

\begin{table}[t]
      \centering
      \setlength{\tabcolsep}{5pt}
    \caption{ICL on dual-\emph{disjoint}-relationship tasks, averaged over 10 repetitions, achieves perfect accuracy when $(p_0, p_1, p_2) \in \{(1/2, 0/ 1,2), (0, 1/2, 1/2), (1/3, 1/3, 1/3)\}$ regardless of architectures and embedding dimensions ($d_E$). When $(p_0, p_1, p_2) = (1/2, 1/2, 0)$, ICL already performs well under the balanced scenario.}
    %\vspace{2mm}
    \begin{tabular}{c c c c c c c}
    \toprule
    \multicolumn{1}{c}{} & \multicolumn{2}{c}{Balanced} & \multicolumn{2}{c}{Imbalanced} & \multicolumn{2}{c}{Extreme}\\
    \cmidrule(lr){2-3} \cmidrule(lr){4-5} \cmidrule(lr){6-7}
    $(p_0, p_1, p_2)$ & $d_E$ = 10 & $d_E$ = 100 & $d_E$ = 10 & $d_E$ = 100 & $d_E$ = 10 & $d_E$ = 100\\
    \midrule
    $(0, 1, 0)$ & (0, 0) & (0, 0) & (0, 0) & (0, 0) & (0, 0) & (0, 0) \\
    $(0, 0, 1)$ & (0, 0) & (0, 0) & (0.16, 0.14) & (0, 0) & (0.21, 0.29) & (0, 0) \\
    $(1/2, 1/2, 0)$ & (1, 1) & (0.82, 0.83) & (0.28, 0.27) & (0.95, 0.95) & (0.83, 0.85) & (0.91, 0.91) \\
    $(1/2, 0, 1/2)$ & (1, 1) & (1, 1) & (1, 1) & (1, 1) & (1, 1) & (1, 1)\\
    $(0, 1/2, 1/2)$ & (1, 1) & (1, 1) & (1, 1) & (1, 1) & (1, 1) & (1, 1) \\
    $(1/3, 1/3, 1/3)$ & (1, 1) & (1, 1) & (1, 1) & (1, 1) & (1, 1) & (1, 1) \\
    \bottomrule
    \end{tabular}
    \label{tab:two-discon}
    \\[-4mm] % Adjust negative space for reducing the gap
\end{table}

% In the first experiment, we focus on 160 countries with a population exceeding one million in 2022. Among these countries, 31 have capital cities that are not their most populous cities, denoted by \textit{type A}. The remaining 129 countries fall under \textit{type B}. Each ICL prompt includes three type A countries among $c_1, \cdots, c_6$ to emphasize that the desired relationship is \texttt{(country)-(capital)} rather than \texttt{(country)-(largest city)}. Subsequently, we randomly generate 1,000 prompts, with 500 having a $c_7$ representing a type A country and 500 having a $c_7$ representing a type B country. The ICL accuracies corresponding to type A and type B prompts are $0.58$ and $0.96$, respectively. 

% In the second experiment, we consider all 50 states, among which 33 are of \textit{type A} and 17 are of \textit{type B}, defined similarly. The ICL accuracies corresponding to type A and type B prompts are found to be $0.69$ and $0.84$, respectively. From both experiments, we notice that LLaMA 2 performs better on type B prompts (i.e., the capital city as the largest city). This suggests that ICL may arise from co-occurrence information, as larger cities tend to appear more frequently compared to smaller ones.

%\vspace{-2mm}

\subsection{Experiments on a synthetic corpus}
\label{sec:synth-exp}
%\vspace{-2mm}
We conduct experiments on a synthetic corpus consisting of \textit{(country)-(capital)} and \textit{(country)-(IOC code)} relationships. Each sentence in the corpus is categorized into exactly one of six possible categories: (1) exactly one country-capital pair; (2) exactly two country-capital pairs; (3) exactly one country-IOC pair; (4) exactly two country-IOC pairs; (5) exactly one country without any pair; and (6) no country. In sentences with country-capital pairs, each capital city can appear in any position relative to the country. Conversely, in sentences with country-IOC pairs, each IOC code must directly follow the country. The data source and corpus generation process are detailed in Appendix \ref{app:exp-details}.

Two models are trained on this corpus: a CBOW and a five-layer two-head autoregressive transformer. Both models have an embedding dimension of $100$. We then compare the ICL accuracies for both relationships given one to five in-context examples. For the CBOW model, the country-capital accuracies are $(0.81, 0.82, 0.78, 0.73, 0.65)$ and the country-IOC accuracies are $(0.15, 0.38, 0.59, 0.71, 0.79)$. Here, the $i$-th number corresponds to the accuracy given $i$ in-context examples. For the transformer, the accuracies are $(0.00, 0.15, 0.34, 0.22, 0.07)$ and $(1.00, 0.77, 0.78, 0.97, 0.99)$, respectively.

\textit{When using the transformer, we find that the accuracies for the country-IOC task are significantly higher compared to those for the country-capital task}. This is likely because each IOC code consistently follows the corresponding country in the corpus, similar to ICL prompts. On the other hand, ICL fails to work on the country-capital task, where there is no consistent pattern in how each pair occurs in the corpus. Meanwhile, \textit{ICL works decently well on both tasks under the CBOW model}.
\section{The essential role of positional information in enabling in-context learning }

\label{sec:icl-pos}

We examine another common example of in-context learning (ICL), where the task involves predicting the first (or second) token in a sequence. This task resembles general logic reasoning tasks that require recognizing patterns that do not typically co-occur in a sentence, such as \textit{(word)-(first letter)}~\citep{xu2024large,chen2024parallel}. While Section \ref{sec:co-occ} shows that positional encoding is irrelevant for ICL in word analogy tasks, positional information proves essential for such logic reasoning tasks. Specifically, we consider a simpler task of modeling $x_{i_1} x_{i_2} x_{i_3} x_{i_1}$. Theorem \ref{prop:icl-pe} underscores the importance of positional information to correctly predict $x_{i_1}$ from $x_{i_1} x_{i_2} x_{i_3}$ in a single-layer model, and provides a construction of an attention-based model achieving zero loss and perfect accuracy on this task. Its proof is in Appendix \ref{app:proof-2}.
\begin{thm}[Necessity of modeling positions]
\label{prop:icl-pe}
    Let the vocabulary be $\mathcal{V} = \{1,2,\cdots,|V|\}$ and the training sequences take the form $x_{i_1} x_{i_2} x_{i_3} x_{i_1}$, where $x_{i_1} \neq x_{i_2} \neq x_{i_3} \neq x_{i_1}$ are chosen uniformly at random from $\mathcal{V}$. Consider a one-layer model that predicts the last $x_{i_1}$ via a learned function $f(\{ x_{i_1}, x_{i_2}\}, x_{i_3})$ using the cross-entropy loss. In this case, it is not possible to achieve pefect accuracy or zero loss. On the other hand, we can achieve zero loss (and thus perfect accuracy) by incorporating positional information, i.e., via a learned function $\tilde{f}(\{ (x_{i_1}, 1), (x_{i_2}, 2)\}, (x_{i_3}, 3))$.
\end{thm}
%\vspace{-2mm}
Here, $f(\{ x_{i_1}, x_{i_2}\}, x_{i_3})$ represents a scenario where the model lacks positional information (e.g., $f$ is a one-layer autoregressive transformer without positional embeddings). Note that the output of this function is identical for inputs $x_{i_1}x_{i_2}x_{i_3}$ and $x_{i_2}x_{i_1}x_{i_3}$, which leads to the impossibility of attaining zero loss. In contrast, $\tilde{f}(\{ (x_{i_1}, 1), (x_{i_2}, 2)\}, (x_{i_3}, 3))$ refers to a scenario where the model has access to positional information. We provide a construction of $\tilde{f}$ that achieves zero loss in Appendix \ref{app:proof-2}. 

\textbf{Experiments.} We validate Theorem \ref{prop:icl-pe} by training transformers with causal masking to autoregressively learn sequences of the form $x_{i_1} x_{i_2} x_{i_3} x_{i_1}$, and assessing their accuracy in predicting the last token on a separate test data of the same pattern. We use $|V| = 20$ and an embedding dimension of $10$. We consider these settings: (i) \textit{number of layers}: 1, 5; (ii) \textit{positional embeddings}: learned, sinusoidal, no positional embeddings; and (iii) \textit{train-test split}: each token in the vocabulary is the first token in both the training and test sets (\textit{Both}), each token in the vocabulary is the first token in either set, but not both (\textit{Either}). More details are provided in Appendix \ref{app:exp-details}.
%\vspace{-1.5mm}
\begin{table}[t]
\centering
\caption{Prediction accuracy with single/multi-layer models. For ICL to occur, the first tokens of training sentences should cover the entire vocabulary (\textit{Both}). Also, positional embeddings are essential, especially in one-layer models.}
\begin{tabular}{c c c c c}
\toprule
\multicolumn{1}{c}{} & \multicolumn{2}{c}{Both} & \multicolumn{2}{c}{Either} \\
\cmidrule(lr){2-3} \cmidrule(lr){4-5}
Pos. emb. & 1-layer & 5-layer & 1-layer & 5-layer \\
\midrule
Learned & 1 & 1 & 0 & 0 \\
Sinusoidal & 1 & 1 & 0 & 0  \\
No pos. emb. & 0.30 & 0.89 & 0 & 0  \\
\bottomrule
\end{tabular}
\label{tab:icl-pos}
\end{table}

\textbf{Results.} Table \ref{tab:icl-pos} summarizes the results. Two main findings emerge: (1) for the model to generalize to unseen sentences, each token in $\mathcal{V}$ should be present as the first token in both the training and test sets; (2) positional embeddings are crucial when using only one~attention~layer. Note that in practice, the condition in (1) is likely met due to the vast size of LLMs' pre-training data.

\textbf{Multiple layers.} Proposition \ref{lem:pe} shows that multi-layer models can encode positional information without explicit positional embeddings.

\begin{prop}[Multi-layer models can encode positions]
\label{lem:pe}
    Consider the sentence $x_{i_1} x_{i_2} x_{i_3} x_{i_1}$. Using a two-layer autoregressive model, the model's final output for predicting the last $x_{i_1}$ is given by $t(x_{i_1} x_{i_2} x_{i_3}) := g_3\left( \{ f_1(\{x_{i_1}\}), f_2(\{x_{i_1}\},x_{i_2}) \}, f_3(\{x_{i_1},x_{i_2}\},x_{i_3}) \right)$ for some $f_1, f_2, f_3$, and $g_3$. 
\end{prop}
%\vspace{-2mm}
The proof is in Appendix \ref{app:proof-pe}. Proposition \ref{lem:pe} shows that we generally have $t(x_{i_1} x_{i_2} x_{i_3}) \neq t(x_{i_2} x_{i_1} x_{i_3})$, unlike in the one-layer case. Consequently, high accuracy is achievable without positional embeddings, as shown in Table~\ref{tab:icl-pos}. This result parallels findings in \citet{haviv2022transformer} that autoregressive transformers implicitly encode positions, even without positional embeddings.

\textbf{Roadmap of Section \ref{sec:icl-pos}.} In the rest of this section, we consider settings where each sentence contains repeating patterns. Section \ref{sec:one-pattern} focuses on a simple scenario where training sentences follow the form \textit{abacdc}, where $a \neq b$ and $c \neq d$, or a noisy variation of it. The ICL prompts maintain the same pattern but use different combinations of \textit{ab} and \textit{cd} from those in the training data. Our goal is to understand what types of training data facilitate ICL in clean or noisy scenarios. Section \ref{sec:two-patterns} explores a more realistic case where two possible patterns are present: repeating the first letter (\textit{abca}) and repeating the second letter (\textit{abcb}). 
%\vspace{-2mm}

\begin{table}[t]
\centering
\caption{ICL on single-pattern tasks, averaged over 10 repetitions, achieves near-perfect accuracy in the clean data scenario regardless of architectures and embedding dimension ($d_E$). The one-noisy scenario is the most challenging, with sinusoidal embeddings giving a higher accuracy. In the block-noisy scenario, learned positional embeddings result in significantly better ICL performance.}
%\vspace{2mm}
\begin{tabular}{c c c c c c c}
\toprule
\multicolumn{1}{c}{} & \multicolumn{3}{c}{$d_E$ = 10} & \multicolumn{3}{c}{$d_E$ = 100} \\
\cmidrule(lr){2-4} \cmidrule(lr){5-7}
Pos. emb. & Clean & One-noisy & Block-noisy & Clean & One-noisy & Block-noisy  \\
\midrule
Learned & 0.97 & 0.00 & 0.95 & 1.00 & 0.00 & 1.00\\
Sinusoidal & 0.66 & 0.10 & 0.01 & 0.96  & 0.00 & 0.55\\
RoPE \citep{su2024roformer} & {0.31} & {0.00} & {0.03} & {0.48} & {0.00} & {0.00} \\
\bottomrule
\end{tabular}
\label{tab:icl-one-pattern}
\end{table}

\subsection{In-context learning on single-pattern tasks}
\label{sec:one-pattern}

In this section, we examine the case where the training sentences follow the pattern \textit{abacdc}. To replicate real-world training scenarios, we also analyze how incorporating nuisance tokens into the training sentences affects the ICL capability of autoregressive models. To formalize the discussion, let our vocabulary be $\mathcal{V} \cup \mathcal{N}$, where $\mathcal{N}$ represents the nuisance tokens. We define $S = \{(a,b) \mid a,b \in \mathcal{V}, a \neq b\}$ and partition $S$ into $S_1$ (for training sentences) and $S_2$ (for ICL prompts). This is to ensure that training sentences are distinct from ICL prompts. Furthermore, we assume $\{c[1] \mid c \in S_1\} = \{c[1] \mid c \in S_2\} = \mathcal{V}$, where $c[i]$ is the $i$-th element of $c$. In other words, each token in $\mathcal{V}$ can be the first token in both the training sentences and ICL prompts. We consider three scenarios: 
\begin{enumerate}[leftmargin=*, noitemsep,topsep=0pt]
    \item \textit{Clean}: Training data follow the form \textit{abacdc} where $ab, cd \in S_1$. ICL prompts follow the form \textit{\underline{abacd}} where \textit{\underline{ab}}, \textit{\underline{cd}} $\in S_2$.
    \item \textit{One-noisy}: Training data follow the form \textit{abacdc} where $ab, cd \in S_1$, with one nuisance token $n \in \mathcal{N}$ randomly inserted anywhere except the last position (to ensure ICL prompts do not resemble the training data). ICL prompts follow the form \textit{\underline{abacd}} where \textit{\underline{ab}}, \textit{\underline{cd}} $\in S_2$.
    \item \textit{Block-noisy}: Training data follow the form \textit{abacdc} where $ab, cd \in S_1$, with three consecutive nuisance tokens $n_1, n_2, n_3 \in \mathcal{N}$ randomly inserted while preserving the \textit{aba} and \textit{cdc} blocks. ICL prompts follow the form \textit{\underline{abacdcef}} where \textit{\underline{ab}}, \textit{\underline{cd}}, \textit{\underline{ef}} $\in S_2$.
\end{enumerate}

We set the vocabulary size $|V| = 20$, the number of nuisance tokens $N = 20$, and use only one attention layer as additional layers do not improve performance. See Appendix \ref{app:exp-details} for more details.

\textbf{Results.} Table \ref{tab:icl-one-pattern} reveals interesting phenomena. First, under the clean data scenario, ICL performs exceptionally well, with an observed performance increase with learned positional embeddings and a larger embedding dimension. However, ICL is notably challenging under the one-noisy scenario. In the block-noisy scenario, learned positional embeddings are crucial for satisfactory ICL performance. Theorem \ref{prop:icl-one-pattern} formalizes these findings.
\begin{thm}[Blocked nuisance token structure facilitates ICL]
\label{prop:icl-one-pattern}
    Consider a sufficiently large autoregressive position-aware model that can achieve the minimum possible theoretical loss. Training this model in the one-noisy (block-noisy) scenario results in zero (perfect) ICL accuracy.
\end{thm}
%\vspace{-2mm}
The proof is in Appendix \ref{app:proof-icl-one-pattern}. Theorem \ref{prop:icl-one-pattern} says that ICL works perfectly under the block-noisy scenario, yet fails to work under the one-noisy scenario. However, as shown in Table \ref{tab:icl-one-pattern}, the use of sinusoidal positional embeddings significantly enhances prediction accuracy in the one-noisy scenario. This may be due to the fact that sinusoidal embeddings can encode relative positional information \citep{vaswani2017attention}. For example, training sentences of the form \textit{nabacdc}, where \textit{n} $\in \mathcal{N}$, may help in predicting the most likely token following the ICL prompt \textit{\underline{abacd}}.

%\vspace{-2mm}
\subsection{In-context learning on dual-pattern tasks}
\label{sec:two-patterns}
We next examine the case where training data and ICL prompts contain two different patterns occurring equally likely: \textit{abcadefd} and \textit{abcbdefe}, where $a,b,c$ and $d,e,f$ are distinct. We consider the \textit{clean} and \textit{block-noisy} scenarios as in Section \ref{sec:one-pattern}, and set $|V| = N = 20$ (details in Appendix \ref{app:exp-details}).

\textbf{Results.} Table \ref{tab:icl-two-patterns} outlines the ICL performance for both scenario types across different model configurations. Unlike the single-pattern scenario, there is an improvement in performance with five layers compared to one layer, particularly with learned positional embeddings.

This phenomenon is related to the notion of \textit{induction heads}, where at least two layers may be necessary to distinguish the two patterns \citep{olsson2022in}. This is reflected in Figure \ref{fig:one-vs-five}, which compares the accuracy trajectories of one-layer and five-layer models. While the five-layer setup effectively differentiates the two patterns, the one-layer configuration fails to do so. Meanwhile, in both clean and block-noisy scenarios, learned positional embeddings lead to notably higher accuracies as compared to sinusoidal ones, similar to the single-pattern case.

\begin{figure}[t]
    \centering
    \begin{subfigure}[t]{0.49\textwidth} % Adjust width as needed
        \centering
        \includegraphics[width=0.92\textwidth]{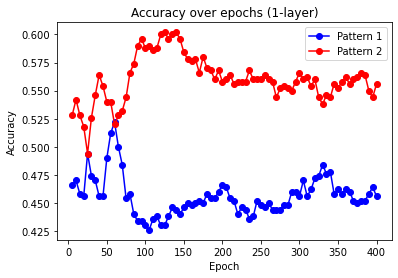}
    \end{subfigure}
    \hfill % This will insert a space between the two figures
    \begin{subfigure}[t]{0.49\textwidth} % Adjust width as needed
        \centering
        \includegraphics[width=0.92\textwidth]{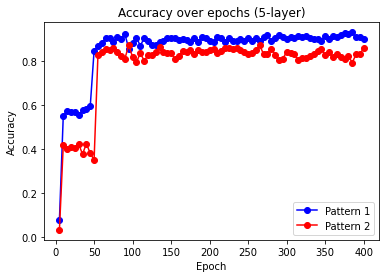}
    \end{subfigure}
    \caption{One-layer models fail to differentiate the two patterns in Section \ref{sec:two-patterns}, as evidenced by the accuracy trajectory graph on the left. On the other hand, five-layer models are capable of doing so.}
    \label{fig:one-vs-five}
    %\vspace{-5mm}
\end{figure}
%\vspace{-4mm}
\section{Scenarios where in-context learning fails}
\label{sec:failed}
%\vspace{-2mm}
In this section, we consider two scenarios where in-context learning (ICL) fails, irrespective of architectures. In Section \ref{sec:failed-1}, we consider a logic reasoning task requiring identification and generalization of a repetition meta-pattern within sequences.  In Section \ref{sec:failed-2}, we explore a word analogy task where relevant word pairs appear in unstructured training sentences but are limited to fixed positions. Section \ref{sec:failed-synth} concludes with a synthetic data experiment supporting the theory.

\begin{table}[t]
\centering
\caption{ICL on dual-pattern tasks, averaged over 10 repetitions, achieves notably better accuracy using learned than sinusoidal embeddings. Near-perfect accuracy is attained in the clean scenario by a 5-layer transformer with an embedding dimension ($d_E$) of 100 and learned positional embeddings. The block-noisy scenario is challenging; the same model attains the best performance.}
%\vspace{2mm}
\begin{tabular}{l c c c c c}
\toprule
& \multicolumn{1}{c}{} & \multicolumn{2}{c}{$d_E$ = 10} & \multicolumn{2}{c}{$d_E$ = 100} \\
\cmidrule(lr){3-4} \cmidrule(lr){5-6}
& Pos. emb. & Clean & Block-noisy & Clean & Block-noisy \\
\midrule
\multirow{2}{*}{1-layer} & Learned & (0.33, 0.33) & (0.15, 0.16) & (0.51, 0.49) & (0.49, 0.50) \\
& Sinusoidal & (0.12, 0.66) & (0.03, 0.03) & (0.51, 0.48) & (0.06, 0.10) \\
\midrule
\multirow{2}{*}{5-layer} & Learned & (0.39, 0.39) & (0.23, 0.22) & (0.97, 0.98) & (0.87, 0.70) \\
& Sinusoidal & (0.32, 0.34) & (0.04, 0.04) & (0.83, 0.82) & (0.04, 0.07) \\
\bottomrule
\end{tabular}
\label{tab:icl-two-patterns}
\end{table}

%\vspace{-2mm}
\subsection{Failed scenario 1: Sentences with repeating patterns}
\label{sec:failed-1}

In this meta-pattern recognition and generalization task, each training sequence follows a repeating pattern based on its starting tokens, and the ICL task sequence requires the model to identify this repetition and extend it to a new, unseen starting pattern. 
Specifically, our training data comprises sentences in the form of \textit{abacdcefe}, where $a \neq b$, $c \neq d$, and $e \neq f$. Note that each sentence is structured into three blocks, each consisting of three tokens with the same pattern. For the ICL task, we consider predicting $\underline{f}$ from the prompt $\underline{abbcddef}$, where $\underline{a} \neq \underline{b}$, $\underline{c} \neq \underline{d}$, and $\underline{e} \neq \underline{f}$. Given the repeated pattern within each training sequence, a well-trained model might be expected to output $\underline{f}$ to continue the pattern established in the in-context examples: $\underline{a} \underline{b} \underline{b}$ and $\underline{c} \underline{d} \underline{d}$. However, as seen in Table \ref{tab:icl-fail}, all models fail to recognize and apply the pattern, resulting in incorrect predictions.

\textbf{Formalization.} We now formalize a generalization of this scenario. Let the vocabulary be $\mathcal{V} = \{1,2,\cdots,|V|\}$, and define $S = \{(a,b) \mid a,b \in \mathcal{V}, a \neq b\}$.  To ensure training sentences are distinct from the ICL prompts, we first partition $S$ into $S_1$ and $S_2$, where $\{c[1] \mid c \in S_1\} = \{c[1] \mid c \in S_2\} = \mathcal{V}$. Here, $c[i]$ denotes the $i$-th element of $c$. Suppose we autoregressively train a sufficiently large position-aware model so that it is possible to achieve the minimum possible theoretical loss. The training sentences take the form $x_{11}x_{12}x_{11}x_{21}x_{22}x_{21} \cdots x_{N1}x_{N2}x_{N1}$, where $x_{i1} \neq x_{i2}$ and $(x_{i1}, x_{i2})$ is independently selected from $S_1$ for every $i \in [N]$. Theorem \ref{prop:failed-1}, whose proof is in Appendix \ref{app:proof-failed-1}, states that ICL fails regardless of the number of in-context examples. 

\begin{thm}[Failure of ICL: Different repeated patterns]
    \label{prop:failed-1}
    Consider the generalized scenario in Section \ref{sec:failed-1}. 
    For any $1 \leq \ell \leq N$, given an in-context prompt of the form  $\underline{x_{11}x_{12}x_{12}x_{21}x_{22}x_{22}} \cdots \underline{x_{\ell1}x_{\ell2}}$ where $\underline{x_{i1}} \neq \underline{x_{i2}}$ and $(\underline{x_{i1}}, \underline{x_{i2}}) \in S_2$ for every $i \in [\ell]$, the model predicts $\underline{x_{\ell1}}$ instead of $\underline{x_{\ell2}}$.
\end{thm}
%\vspace{-2mm}
\textbf{Results.} Theorem \ref{prop:failed-1} and Table \ref{tab:icl-fail} demonstrate that ICL achieves zero accuracy irrespective of the number of in-context examples ($\ell - 1$). This insight sheds light on the ICL capacity of autoregressive models. Simply put, if the pattern in the in-context examples differs significantly from any pattern in the training data, ICL may not occur. These results align with the findings of \citet{raventos2023pretraining} and \citet{yadlowsky2023pretraining} on the importance of data diversity for ICL.

%\vspace{-2mm}
\subsection{Failed scenario 2: Sentences with co-occurring word pairs restricted to fixed locations}
\label{sec:failed-2}

\begin{table}[t]
\centering
\caption{ICL in failed scenarios, averaged over 10 repetitions, achieves zero accuracy for any architecture and embedding dimension ($d_E$).}
\begin{tabular}{l c c c c c}
\toprule
& \multicolumn{1}{c}{} & \multicolumn{2}{c}{Failed scenario 1} & \multicolumn{2}{c}{Failed scenario 2} \\
\cmidrule(lr){3-4} \cmidrule(lr){5-6}
& Pos. emb. & $d_E$ = 10 & $d_E$ = 100 & $d_E$ = 10 & $d_E$ = 100 \\
\midrule
\multirow{2}{*}{1-layer} & Learned & 0.00 & 0.00 & 0.01 & 0.00 \\
& Sinusoidal & 0.01 & 0.00 & 0.00 & 0.00 \\
\midrule
\multirow{2}{*}{5-layer} & Learned & 0.00 & 0.00 & 0.00 & 0.00 \\
& Sinusoidal & 0.00 & 0.00 & 0.00 & 0.00 \\
\bottomrule
\end{tabular}
\label{tab:icl-fail}
\end{table}
We revisit the word analogy task in Section \ref{sec:co-occ}.  The training data now comprises sentences of the form of $a_i pqrs b_i$, where $(a_i, b_i)$ represents a frequently co-occurring word pair and $p,q,r,s$ represent other words. For the ICL task, we consider predicting $b_{i_3}$ from the prompt $a_{i_1} b_{i_1} a_{i_2} b_{i_2} a_{i_3}$, where $i_1, i_2, i_3$ are distinct. As each training sentence always contains an $(a_i, b_i)$ pair at a fixed location, we expect a well-trained model to output $b_{i_3}$ to maintain the pattern in in-context examples: $a_{i_1} b_{i_1}$ and $a_{i_2} b_{i_2}$. Yet Table \ref{tab:icl-fail} shows none of the models can identify the patterns and predict the correct token.

\textbf{Formalization.} We now formalize a generalization of this scenario. Let the vocabulary be $\{(a_i, b_i)\}_{i \in [I]} \cup \mathcal{V}$, where $\mathcal{V} = \{1,2,\cdots,|V|\}$ represent other words. As in Section \ref{sec:failed-1}, we autoregressively train a sufficiently large position-aware model that can achieve the minimum possible theoretical loss. The training sentences take the form 
$a_{i} v_1 v_2 \cdots v_{2k} b_{i}$, where $i$ and $v_{1:2k}$ are independently chosen from $[I]$ and $\mathcal{V}$, respectively, uniformly at random. Theorem \ref{prop:failed-2}, whose proof is in Appendix \ref{app:proof-failed-2}, states that ICL fails regardless of the number of in-context examples. 

\begin{thm}[Failure of ICL: Different pattern structures]
    \label{prop:failed-2}
    Consider the generalized scenario in Section \ref{sec:failed-2}. For any $1 \leq \ell \leq k + 1$, given an in-context prompt of the form $a_{i_1}b_{i_1}a_{i_2}b_{i_2} \cdots a_{i_\ell}$ with distinct $i_j$'s, the model never predicts $b_{i_\ell}$: it predicts a uniform probability vector over $\mathcal{V}$ when $1 \leq \ell \leq k$, and $b_{i_1}$ when $\ell = k+1$.
\end{thm}
%\vspace{-2mm}
\textbf{Results.} Theorem \ref{prop:failed-2} highlights the finding that the success of ICL relies heavily on how the patterns appear in the training data. In this scenario, the $(a_i, b_i)$ pairs consistently appear at the beginning and end of each training sentence, and we anticipate the model to recognize this relationship for ICL to occur. However, as shown in Theorem \ref{prop:failed-2} and Table \ref{tab:icl-fail}, this is not the case.

%\vspace{-2mm}
\subsection{Experiment on a synthetic corpus}
\label{sec:failed-synth}

We conduct an experiment on a synthetic corpus featuring \textit{(country)-(capital)} relationships. Each sentence falls into one of four categories: (1) exactly one country-capital pair, (2) exactly two country-capital pairs, (3) a single country without a pair, and (4) no country. In sentences with one country-capital pair, the capital appears in the first position, the country in the last, and each sentence contains six words (as in Section \ref{sec:failed-2}). The corpus generation process is detailed in Appendix \ref{app:exp-details}. 

We train a five-layer two-head autoregressive transformer on this corpus, with an embedding dimension of $100$. Similar to Section \ref{sec:synth-exp}, we evaluate the ICL accuracies using prompts involving countries and their capitals. The results show zero ICL accuracy across varying in-context examples (one to five), supporting our theory.

\section{Discussion}
This paper examines how in-context learning (ICL) arises from pre-training on unstructured language data, with three key findings: (1) ICL for word analogy tasks can emerge from simple co-occurrence modeling, using models like continuous bag of words (CBOW) without positional encoding or attention; (2) positional information and structured nuisance tokens are essential for ICL in logic reasoning tasks that require recognizing rare patterns and generalizing to new tokens; and (3) the structure of training data significantly impacts ICL effectiveness.

\section{Limitations and future work}

This study has several limitations. Firstly, the experiments are conducted on a relatively small scale. However, they still provide sufficient evidence to support the theoretical findings. Secondly, the focus of this study is on specific types of in-context learning (ICL) tasks, as described in Section \ref{sec:intro}. Thirdly, the pre-training data considered in this work may not match the valid grammatical sentences that language models are usually trained on. Nonetheless, our co-occurrence results still apply to grammatical sentences, as the co-occurring pairs can appear naturally within them (e.g., "Beijing is the capital of China," or "the city of Beijing is located in China"). Lastly, real data sets are not utilized due to the lack of alignment with the study objectives. 

Despite these limitations, we believe this work provides valuable understanding of the key factors enabling ICL to occur from training on unstructured natural language data, supported by both theoretical and empirical evidence from experiments involving prompting and synthetic data. Further analyses on other ICL tasks and their reliance on model architecture can be fruitful avenues for future work.

\vspace{20pt}

\textbf{Acknowledgements. } This work was supported in part by the Office of Naval Research under
grant number N00014-23-1-2590, the National Science Foundation under
grant numbers 2231174 and 2310831, No. 2428059, and a Michigan Institute
for Data Science Propelling Original Data Science (PODS) grant.

\clearpage
%\putbib[attention]

%\end{bibunit}

%\bibliographystyle{alp}
\bibliography{icl}

\clearpage
\appendix

% !TeX root = ./main.tex

\begin{center}
\textbf{\Large Supplementary Material}
\end{center}
\section{Related work}
\label{app:rel-work}

Large language models (LLMs), such as transformers, are widely recognized for their outstanding performance in in-context learning (ICL) \citep{brown2020language}. ICL refers to the capability of LLMs to discern specific tasks and generate predictions based on prompt exemplars without needing any parameter updates. A multitude of studies have been dedicated to exploring this intriguing phenomenon from various theoretical and empirical perspectives. In this section, we provide a brief summary of some of these studies.

Some studies adopted a Bayesian approach to studying ICL. \citet{xie2021an} posited that ICL can be viewed as implicit Bayesian inference. They demonstrated that LLMs can infer a latent document-level concept for next-token prediction during pre-training and a shared latent concept across input-output pairs in an ICL prompt, under the assumption that documents are generated from hidden Markov models (HMMs). \citet{wang2023large} and \citet{zhang2023what} expanded on this idea by exploring more realistic latent variable models beyond HMMs. \citet{wang2023large} argued that large language models function as latent variable models, with latent variables containing task-related information being implicitly inferred. \citet{zhang2023what} showed that without updating the neural network parameters, ICL can be interpreted as Bayesian model averaging parameterized by the attention mechanism. \citet{ahuja2023in} provided empirical evidence that transformers behave like Bayesian predictors when performing ICL with linear and non-linear function classes. \citet{dalal2024the} proposed a Bayesian learning framework to understand ICL through the lens of text generation models represented by multinomial transition probability matrices. \citet{chiang2024understanding} proposed the pelican soup framework to explain ICL without relying on latent variable models. This framework incorporates concepts such as a common sense knowledge base, natural language classification, and meaning association, enabling the establishment of a loss bound for ICL that depends on the number of in-context examples.

\citet{garg2022what} formulated ICL as learning a specific function class $\mathcal{F}$ from prompts of the form $\left(x_1, f(x_1), \ldots, x_n, f(x_n), x_{n+1}\right)$ and their corresponding responses $f(x_{n+1})$. Here, $f \in \mathcal{F}$, where $\mathcal{F}$ is a function class. In this context, ICL refers to the capability of a transformer to output a number close to $g(y_{n+1})$ given a prompt of the form $\left(y_1, g(y_1), \ldots, y_n, g(x_n), y_{n+1}\right)$, where $g \in \mathcal{F}$. Many studies adopted this regression formulation of ICL, with some linking ICL to gradient descent. \citet{akyurek2022learning, oswald2023transformers}, and \citet{dai2023why} proved that transformers are capable of implementing gradient descent, which results in their ICL ability. \citet{bai2023transformers} established generalization bounds for ICL and proved that transformers can perform algorithm selection like statisticians. \citet{zhang2024trained} showed that the gradient flow dynamics of transformers converge to a global minimum that enables ICL. \citet{huang2023in} investigated the learning dynamics of single-layer softmax transformers trained via gradient descent to perform ICL on linear functions. \citet{ahn2024transformers} explored the optimization landscape of transformers and proved that the optimal parameters coincide with an iteration of preconditioned~gradient~descent. 

In a related exploration, \citet{li2023the} showed that softmax regression models learned through gradient descent are similar to transformers. \citet{ren2023in} related ICL with softmax transformers to contrastive learning, where the inference process of ICL can be viewed as a form of gradient descent. \citet{mahankali2023one} proved that minimizing the pre-training loss is equivalent to a step of gradient descent in single-layer linear transformers. \citet{vladymyrov2024linear} established that linear transformers execute a variant of preconditioned gradient descent by maintaining implicit linear models. On the other hand, some studies argued that the ICL ability of transformers cannot be attributed to gradient descent. \citet{fu2023transformers} showed that ICL for linear regression tasks arises from higher-order optimization techniques like iterative Newton's method rather than gradient descent. \citet{wibisono2023on} demonstrated that transformers can perform ICL on unstructured data whose prompt exemplars lack explicit pairings, with softmax attention playing an important role especially when using a single attention layer. \citet{shen2023do} provided empirical evidence that the equivalence between gradient descent and ICL might not be applicable in real-world scenarios. \textit{In contrast to these studies, our work provides a connection between ICL and classical language models like continuous bag of words (CBOW). Specifically, we show that ICL 
for word analogy tasks with semantically related word pairs can arise by modeling co-occurrence patterns via CBOW.}

Numerous studies focused on the pre-training aspects (e.g., data distribution and task diversity) of ICL. \citet{min2022rethinking} showed that the input-label mapping in the in-context examples does not significantly affect ICL performance. \citet{chan2022data} demonstrated that the ICL capabilities of transformers depend on the training data distributions and model features. \citet{kossen2023in} established that ICL considers in-context label information and is capable of learning entirely new tasks in-context. \citet{li2023finding} introduced an iterative algorithm designed to enhance ICL performance by selecting a small set of informative examples that effectively characterize the ICL task. \citet{qin2023in} proposed a method based on zero-shot chain-of-thought reasoning for selecting ICL examples, emphasizing the importance of choosing diverse examples that are strongly correlated with the test sample. \citet{han2023understanding} studied ICL by identifying a small subset of the pre-training data that support ICL via gradient-based methods. They discovered that this supportive pre-training data typically consist of more uncommon tokens and challenging examples, characterized by a small information gain from long-range context. \citet{peng2024revisiting} proposed a selection method for ICL demonstrations that are both data-dependent and model-dependent. \citet{van2024in} introduced a demonstration selection method that enhances ICL performance by analyzing the influences of training samples using influence functions.

In a similar vein, \citet{wu2023how} demonstrated that pre-training single-layer linear attention models for ICL on linear regression with a Gaussian prior can be effectively accomplished with a minimal number of independent tasks, regardless of task dimension. \citet{raventos2023pretraining} emphasized a task diversity threshold that differentiates the conditions under which transformers can successfully address unseen tasks. \citet{yadlowsky2023pretraining} attributed the impressive ICL capabilities of transformers to the diversity and range of data mixtures in their pre-training, rather than their inductive biases for generalizing to new tasks. \citet{ding2023causallm} compared the ICL performance of transformers trained with prefixLM (where in-context samples can attend to all tokens) versus causalLM (where in-context samples cannot attend to subsequent tokens), finding that the latter resulted in poorer ICL performance. \citet{chen2024parallel} discovered that the ICL capabilities of language models rely on the presence of pairs of phrases with similar structures within the same sentence. \citet{zhao2024noisyicl} proposed a calibration scheme that modifies model parameters by adding random noises, resulting in fairer and more confident predictions. \citet{abbas2024enhancing} demonstrated that the ICL predictions from transformer-based models often exhibit low confidence, as indicated by high Shannon entropy. To address this issue, they introduced a straightforward method that linearly calibrates output probabilities, independent of the model's weights or architecture. \textit{Similar to these works, our work highlights  the importance of co-occurrence, positional information, and training data structure for ICL to arise.}

Other studies analyzed ICL from a learning theory perspective. \citet{hahn202a} proposed an information-theoretic bound that explains how ICL emerges from next-token prediction. \citet{wies2023the} derived a PAC-type framework for ICL and finite-sample complexity results. \citet{jeon2024an} introduced a novel information-theoretic view of meta-learning (including ICL), allowing for the decomposition of errors into three components. They proved that in ICL, the errors decrease as the number of examples or sequence length increase. Other studies focus on the mechanistic interpretability component of ICL. \citet{olsson2022in} argued that transformers can develop induction heads that are able to complete token sequences such as [A][B] $\cdots$ [A] $\rightarrow$ [B], leading to impressive ICL performance. \citet{bietti2023birth} examined a setup where tokens are generated from either global or context-specific bigram distributions to distinguish between global and in-context learning. They found that global learning occurs rapidly, while in-context learning is achieved gradually through the development of an induction head. \citet{ren2024identifying} identified semantic induction heads that increase the output logits of tail tokens when attending to head tokens, providing evidence that these heads could play a vital role in the emergence of ICL. \citet{yu2024how} showed that the ICL ability of transformers arises from the utilization of in-context heads, where each query and key matrix collaborate to learn the similarity between the input text and each demonstration example.

A number of works delved into specific data generating processes to provide insight into the emergence of ICL. \citet{bhattamishra2023understanding} examined the ICL ability of transformers by focusing on discrete functions. Specifically, they showed that transformers perform well on simpler tasks, struggle with more complex tasks, and can learn more efficiently when provided with examples that uniquely identify a task. \citet{guo2023how} investigated ICL in scenarios where each label is influenced by the input through a potentially complex yet constant representation function, coupled with a unique linear function for each instance. \citet{akyurek2024in} studied ICL of regular languages produced by random finite automata. They compared numerous neural sequence models and demonstrated that transformers significantly outperform RNN-based models because of their ability to develop \textit{n-gram heads}, which are a generalization of \textit{induction heads}. \citet{sander2024how} analyzed simple first-order autoregressive processes to gain insight into how transformers perform ICL to predict the next tokens. \textit{On the other hand, our work focuses on data generating processes containing semantically related word pairs and repeated token patterns to better understand several components that are crucial for ICL to occur from training on unstructured data.}

Some studies explored how different components of transformers affect their ICL abilities. \citet{ahuja2023a} compared the ICL performance of transformers and MLP-based architectures under distribution shifts. Their findings demonstrate that while both methods perform well in in-distribution ICL, transformers exhibit superior ICL performance when faced with mild distribution shifts. \citet{collins2024in} showed that softmax attention outperforms linear attention in ICL due to its ability to calibrate its attention window to the Lipschitzness of the pre-training tasks. \citet{xing2024benefits} focused on linear regression tasks to identify transformer components that enable ICL. They found that positional encoding is crucial, along with the use of multiple heads, multiple layers, and larger input dimensions. \citet{cui2024superiority} proved that multi-head attention outperforms single-head attention in various practical scenarios, including those with noisy labels and correlated features. \citet{chen2024training} investigated the ICL dynamics of a multi-head softmax attention model applied to multi-task linear regression. They proved the convergence of the gradient flow and observed the emergence of a \textit{task allocation} phenomenon, where each attention head specializes in a specific task.

Finally, several studies proposed various hypotheses on the emergence of ICL and provided theoretical justifications. \citet{swaminathan2023schema} introduced clone-structured causal graphs (CSCGs) to explain how ICL can generalize to unseen sentences via a mechanism called rebinding. \citet{li2023transformers} viewed ICL as an algorithm learning problem where a transformer implicitly constructs a hypothesis function at inference time. \citet{han2023explaining} argued that the ability of transformers  to execute ICL is attributable to their capacity to simulate kernel regression. \citet{singh2023the} explored the interaction between ICL and in-weights learning (IWL) using synthetic data designed to support both processes. They observed that ICL initially emerges, followed by a transient phase where it disappears and gives rise to IWL. \citet{yan2023understanding} studied ICL from the perspective that token co-occurrences play a crucial role in guiding the learning of surface patterns that facilitates ICL. \citet{abernethy2024a} showed that transformers can execute ICL by dividing a prompt into examples and labels, then employing sparse linear regression to deduce input-output relationships and generate predictions. \citet{lin2024dual} developed a probabilistic model that can simultaneously explain both task learning and task retrieval aspects of ICL. Here, task learning refers to the ability of language models to identify a task from in-context examples, while task retrieval pertains to their ability to locate the relevant task within the pre-training data.

\newpage

\section{Proof of Theorem \ref{prop:icl-cbow}}
\label{app:proof-1}
\begin{proof}
Let $|V| = 2K + L$ denote the vocabulary size. Consider a sentence $X$ represented by its one-hot encoding (i.e., $X \in \{0,1\}^{|V| \times S}$). For every position $i \in [S]$, the loss for predicting the word in the $i$-th position given all the other words is given by
$|| AX(\mathds{1}_S - e_i) - X e_i||_2^2,$
where $A = \frac{U^\top V}{S-1} \in \mathbb{R}^{|V| \times |V|}$ and $e_i \in \mathbb{R}^S$ is a zero vector with $1$ on its $i$-th entry. Here, $U$ ($V$) is a matrix consisting of the center (context) embeddings of all tokens, and $A$ is a matrix summarizing the similarity between each pair of words (one as a center word and the other as a context word). Our objective is to find $A$ that minimizes the sum of losses for each position in each sentence. Lemma \ref{lem:1} gives a closed-form expression of the minimizer.

\begin{lemma}
\label{lem:1}
    The minimizer of the overall loss is given by $A = B \left((S-2)B + C \right)^{-1}$.
    Here, $B$ is a matrix whose $(i,j)$-th entry is $p(i,j)$, the probability that for a given (center, context) pair, the center is $i \in |V|$ and the context is $j \in |V|$. Moreover, $C$ is a diagonal matrix whose $i$-th diagonal entry is $p(i) = \sum_{j \in |V|} p(i, j)$. 
\end{lemma}

\begin{proof}
    Let $\mathcal{L}(X) = \sum_{i=1}^S || AX(\mathds{1}_S - e_i) - X e_i||_2^2$ denote the sum of the losses corresponding to all tokens in sentence $X$. By direct calculation,
    
    $$\frac{\partial \mathcal{L}(X)}{\partial A} = 2AX \left(\sum_{i=1}^S (\mathds{1}_S - e_i)(\mathds{1}_S - e_i)^\top \right) X^\top - 2X \left( \sum_{i=1}^S e_i(\mathds{1}_S - e_i)^\top \right) X^\top$$.

Note that $\sum_{i=1}^S (\mathds{1}_S - e_i)(\mathds{1}_S - e_i)^\top = (S-2)\mathds{1}_{S \times S} + \mathbb{I}_{S \times S}$ and $\sum_{i=1}^S e_i (\mathds{1}_S - e_i)^\top = \mathds{1}_{S \times S} - \mathbb{I}_{S \times S}$. Now, let our sentences be $X_1, X_2, \cdots, X_N$. The minimizer of the overall loss thus satisfies

\begin{equation}
    \label{eq:min-loss}
    A \hspace{1mm} \frac{1}{N} \sum_{k=1}^N X_k \left( (S-2)\mathds{1}_{S \times S} + \mathbb{I}_{S \times S}\right) X_k^\top = \frac{1}{N} \sum_{k=1}^N X_k \left( \mathds{1}_{S \times S} - \mathbb{I}_{S \times S}\right) X_k^\top.
\end{equation}

We denote the number of (center, context) pairs across all sentences in which the center is $i$ and the context is $j$  by $\#(i,j)$. Moreover, we define $\#(i) = \sum_{j \in |V|} \#(i,j)$. It is easy to see that Equation (\ref{eq:min-loss}) can be rewritten as

$$A \left( (S-2) \tilde{B} + \tilde{C}\right) = \tilde{B},$$

where $\tilde{B}$ is a matrix such that its $(i,j)$-th entry is $\frac{\#(i,j)}{N}$ and $\tilde{C}$ is a diagonal matrix such that its $i$-th diagonal element is $\frac{\#(i)}{N}$. As $N \rightarrow \infty$, an application of the law of large numbers yields $\frac{\#(i,j)}{N} \rightarrow S(S-1)p(i,j)$ almost surely and $\frac{\#(i)}{N} \rightarrow S(S-1)p(i)$ almost surely, where $p(i,j)$ is the probability that for a given (center, context) pair, the center is $i$ and the context is $j$, and $p(i) = \sum_{j \in |V|} p(i, j)$. 

Thus, as $N \rightarrow \infty$, we have

$$A = B \left((S-2)B + C \right)^{-1},$$

where $B$ and $C$ are defined in the statement of Lemma \ref{lem:1}. \end{proof}

We now define
\begin{itemize}
    \item $p_1 = p(c_i, c_j) = p(d_i, d_j) = p(c_i, d_j) = p(d_i, c_j)$ for any $i \neq j$;
    \item $p_2 = p(r_i, r_j)$ for any $i \neq j$;
    \item $p_3 = p(c_i, d_i) = p(d_i, c_i)$ for any $i$; 
    \item $p_4 = p(c_i, r_j) = p(d_i, r_j) = p(r_j, c_i) = p(r_j, d_i)$ for any $i,j$,
\end{itemize}

where the equalities in the probabilities are a consequence of the data distribution. 

For ease of presentation, we denote a square matrix with $\alpha$ on the diagonal and $\beta$ off the diagonal as $X_{\alpha, \beta}$, and a matrix with all entries $\gamma$ as $Y_\gamma$. We then have

\[
B = \begin{bmatrix}
    X_{0,p_1} & X_{p_3,p_1} & Y_{p_4} \\
    X_{p_3,p_1} & X_{0,p_1} & Y_{p_4} \\
    Y_{p_4} & Y_{p_4} & X_{0, p_2}
\end{bmatrix}.
\]

Now, define $a = (S-2) p_1$, $b = (S-2) p_2$, $c = (S-2) p_3$, $d = (S-2) p_4$, $e = 2(K-1)p_1 + p_3 + Lp_4$, and $f = (L-1)p_2 + 2Kp_4$. It is easy to see that 

\[
(S-2)B + C = \begin{bmatrix}
    X_{e, a} & X_{c, a} & Y_{d} \\
    X_{c,a} & X_{e, a} & Y_{d} \\
    Y_{d} & Y_{d} & X_{f, b}
\end{bmatrix}.
\]

Moreover, its inverse can be written as 

\[
((S-2)B + C)^{-1} = \begin{bmatrix}
    X_{q_5, q_1} & X_{q_3, q_1} & Y_{q_4} \\
    X_{q_3, q_1} & X_{q_5, q_1} & Y_{q_4} \\
    Y_{q_4} & Y_{q_4} & X_{q_6, q_2}
\end{bmatrix},
\]

where

$\Delta = 2a(K-1)(b(L-1)+f) + b(L-1)(c+e) + cf - 2d^2 KL + ef$,
    
$q_1 = -\left(\frac{-a b L + a b - a f + d^2 L}{(2a - c - e)\Delta}\right)$,
    
$q_2 = \frac{2ab(K-1) + b(c+e) - 2d^2K}{(b-f)\Delta}$,
    
$q_3 = -\left(\cfrac{\splitfrac{-2a^2b(K-1)(L-1) - 2a^2f(K-1)+2abc(K-2)(L-1)+2acf(K-2)}{+2(a-c)d^2KL+bc(c+e)(L-1)+cf(c+e)+d^2L(c-e)}}{(c-e)(2a-c-e)\Delta}\right)$,
    
$q_4 = -\left(\frac{d}{\Delta}\right)$,
    
$q_5 = -\left(\cfrac{\splitfrac{-2a^2b(K-1)(L-1) - 2a^2f(K-1)+2abe(K-2)(L-1)+2aef(K-2)}{+2(a-e)d^2KL+be(c+e)(L-1)+ef(c+e)+d^2L(e-c)}}{(e-c)(2a-c-e)\Delta}\right)$,
    
and $q_6 = -\left(\frac{2a(K-1)(b(L-2)+f) + b(L-2)(c+e)+cf-2d^2KL + 2d^2K + ef}{(b-f)\Delta}\right)$.

By computing $A = B((S-2)B+C)^{-1}$, given the following center words, the similarities between them and all possible context words are as follows:

\begin{itemize}
    \item Center word = $c_i$ for any $i$
    \begin{itemize}
        \item $c_i: 2(K-1)p_1q_1 + p_3q_3 + Lp_4q_4$;
        \item $c_j: 2(K-2)p_1q_1 + p_1q_5 + p_3q_1 + p_1q_3 + Lp_4q_4$ ($j \neq i$); 
        \item $d_i: 2(K-1)p_1q_1 + p_3q_5 + Lp_4q_4$;
        \item $d_j: 2(K-2)p_1q_1 + p_1q_3 + p_3q_1 + p_1q_5 + Lp_4q_4$ ($j \neq i$); 
        \item $r_j: 2(K-1)p_1q_4 + p_3q_4 + p_4q_6 + (L-1)p_4q_2$ (for any $j$).
    \end{itemize}
    \item Center word = $d_i$ for any $i$
    \begin{itemize}
        \item $d_i: 2(K-1)p_1q_1 + p_3q_3 + Lp_4q_4$;
        \item $d_j: 2(K-2)p_1q_1 + p_1q_5 + p_3q_1 + p_1q_3 + Lp_4q_4$ ($j \neq i$); 
        \item $c_i: 2(K-1)p_1q_1 + p_3q_5 + Lp_4q_4$;
        \item $c_j: 2(K-2)p_1q_1 + p_1q_3 + p_3q_1 + p_1q_5 + Lp_4q_4$ ($j \neq i$); 
        \item $r_j: 2(K-1)p_1q_4 + p_3q_4 + p_4q_6 + (L-1)p_4q_2$ (for any $j$).
    \end{itemize}
    \item Center word = $r_i$
    \begin{itemize}
        \item $c_j: 2(K-1)p_4q_1 + p_4q_5 + p_4q_3 + (L-1)p_2q_4$ (for any $j$);
        \item $d_j: 2(K-1)p_4q_1 + p_4q_5 + p_4q_3 + (L-1)p_2q_4$ (for any $j$);
        \item $r_i: 2Kp_4q_4 + (L-1)p_2q_2$;
        \item $r_j: 2Kp_4q_4 + (L-2)p_2q_2 + p_2q_6$ ($j \neq i$).
    \end{itemize}
\end{itemize}

Recall that the ICL problem of interest is the following: given context words $c_{i_1} d_{i_1} \cdots c_{i_{\ell}} d_{i_{\ell}} c_{i_{\ell+1}}$, we aim to predict $d_{i_{\ell+1}}$. Without loss of generality, we can rewrite the problem to predict $d_{\ell + 1}$ given context words $c_1 d_1 \cdots c_\ell d_\ell c_{\ell+1}$. We now compute the total similarity for each possible center word, where $\epsilon^\top \delta$ indicates the similarity between the word $\epsilon$ in the center and the word $\delta$ in the context. 
\begin{itemize}
    \item $c_1$ (or any of $c_2, \cdots, c_\ell$) $: c_1^\top c_1 + \ell c_1^\top c_2 + c_1^\top d_1 + (\ell - 1) c_1^\top d_2$;
    \item $d_1$ (or any of $d_2, \cdots, d_\ell$) $: c_1^\top d_1 + \ell c_1^\top d_2 + c_1^\top c_1 + (\ell - 1) c_1^\top c_2$;
    \item $r_1$ (or any other $r_k$'s) $:(\ell+1)r_1^\top c_1 + \ell r_1^\top d_1 = (2\ell + 1) r_1^\top c_1$;
    \item $c_{\ell+1}: \ell c_1^\top c_2 + \ell c_1^\top d_2 + c_1^\top c_1$;
    \item $d_{\ell+1}: \ell c_1^\top d_2 + \ell c_1^\top c_2 + c_1^\top d_1$;
    \item $c_{\ell+2}$ (or any $c_k$'s not in the context prompt) $: (\ell+1) c_1^\top c_2 + \ell c_1^\top d_2$;
    \item $d_{\ell+2}$ (or any $d_k$'s not in the context prompt) $: (\ell+1) c_1^\top d_2 + \ell c_1^\top c_2$.
\end{itemize}

Note that correctly predicting $d_{\ell+1}$ is equivalent to the following conditions being simultaneously satisfied:
\begin{itemize}
    \item $c_1^\top d_1 > c_1^\top c_1$, equivalent to $p_3 q_5 > p_3 q_3$;
    \item $c_1^\top d_2 > c_1^\top c_1$ and $c_1^\top c_2 > c_1^\top c_1$, equivalent to $p_1q_3 + p_3q_1 + p_1q_5 > 2p_1q_1 + p_3q_3$;
    \item $c_1^\top d_1 > c_1^\top c_2$ and $c_1^\top d_1 > c_1^\top d_2$, equivalent to $2p_1q_1 + p_3q_5 \geq p_1q_5 + p_1q_3 + p_3q_1$;
    \item $2\ell c_1^\top c_2 + c_1^\top d_1 > (2\ell+1)r_1^\top c_1$, equivalent to $2\ell (2(K-2)p_1q_1 + p_1q_5 + p_3q_1 + p_1q_3 + Lp_4q_4) + 2(K-1)p_1q_1 + p_3q_5 + Lp_4q_4 > (2\ell+1)(2(K-1)p_4q_1 + p_4q_5 + p_4q_3 + (L-1)p_2q_4)$;
\end{itemize}

In our data generating process, it is easy to see that $p_1 = 0$, $p_2 = \frac{(S-2)(S-3)}{L(L-1)}$, $p_3 = \frac{1}{K}$, and $p_4 = \frac{S-2}{KL}$, where each $p_i$ is multiplied by a constant $S(S-1) > 0$ (without loss of generalization) to make calculations easier. From here, we have $a = 0$, $b = \frac{(S-2)^2 (S-3)}{L(L-1)}$, $c = \frac{S-2}{K}$, $d = \frac{(S-2)^2}{KL}$, $e = \frac{S-1}{K}$, and $f = \frac{(S-1)(S-2)}{L}$. Substituting to the above, we have
\begin{itemize}
    \item $q_1 = \frac{(S-2)^4}{\Delta KL (2S-3)}$;
    \item $q_3 = \frac{-K(S-2)^2(S-1)^2 - (S-2)^4}{\Delta KL (2S-3)}$;
    \item $q_4 = \frac{-(2S-3)(S-2)^2}{\Delta KL (2S-3)}$;
    \item $q_5 = \frac{K(S-2)(S-1)^3 + (S-2)^4}{\Delta KL (2S-3)}$,
\end{itemize}
where $\Delta = \frac{(S-1)^2 (S-2)}{KL} > 0$.

We now check when these conditions are simultaneously satisfied. The first condition is equivalent to $p_3 > 0$ and $K > \frac{2(S-2)^3}{(S-1)^2 (2S-3)}$, which always hold. The second condition reduces to $p_3 > 0$ and $2(S-2)^4 + K(S-2)^2(S-1)^2 > 0$, which is also true. The third condition can be written as $p_3 > 0$ and $K(S-2)(S-1)^3 > 0$, which always hold. The last condition becomes

$$(2\ell+1)((K+L)(S-2)^2(S-1)+K(S-2)(S-1)^2-2(S-2)^4) < KL(S-1)^3,$$

which is equivalent to

$$2\ell + 1 < \frac{KL(S-1)^3}{(K+L)(S-2)^2(S-1)+K(S-2)(S-1)^2-2(S-2)^4},$$

completing the proof. 

Note that this condition ensures that the model predicts $d_{\ell + 1}$ instead of one of the $r_i$'s.
\end{proof}

\newpage
\section{Comparison of ICL performance using squared and cross-entropy loss across different numbers of examples}
\label{app:comp-loss-ex}

\begin{table}[h]
    \centering
    \caption{ICL performance in the \textit{clean} scenario, evaluated with both squared and cross-entropy loss functions across different numbers of examples (0 to 8) with $d_E = 100$, averaged over 10 repetitions.}
    \begin{tabular}{c p{0.7cm} p{0.7cm} p{0.7cm} p{0.7cm} p{0.7cm} c p{0.7cm} p{0.7cm} p{0.7cm} p{0.7cm} p{0.7cm}}
    \toprule
    \multicolumn{1}{c}{} & \multicolumn{5}{c}{Squared} & \multicolumn{1}{c}{} & \multicolumn{5}{c}{Cross-entropy} \\
    \cmidrule(lr){2-6} \cmidrule(lr){8-12}
    $(p_0, p_1, p_2)$ & 0 & 2 & 4 & 6 & 8 & & 0 & 2 & 4 & 6 & 8 \\
    \midrule
    $(0, 1, 0)$ & 1 & 1 & 0 & 0 & 0 & & 0.87 & 0 & 0 & 0 & 0 \\
    $(0, 0, 1)$ & 1 & 1 & 1 & 0 & 0 & & 1 & 0 & 0 & 0 & 0 \\
    $(1/2, 1/2, 0)$ & 1 & 1 & 1 & 1 & 1 & & 1 & 1 & 0.34 & 0 & 0 \\
    $(1/2, 0, 1/2)$ & 1 & 1 & 1 & 1 & 1 & & 1 & 1 & 1 & 1 & 1 \\
    $(0, 1/2, 1/2)$ & 1 & 1 & 1 & 1 & 1 & & 1 & 1 & 1 & 0 & 0 \\
    $(1/3, 1/3, 1/3)$ & 1 & 1 & 1 & 1 & 1 & & 1 & 1 & 1 & 1 & 0 \\
    \bottomrule
    \end{tabular}
    \label{tab:clean-sq-ce}
\end{table}

\begin{table}[h]
    \centering
    \caption{ICL performance in the \textit{corrupted} scenario, evaluated with both squared and cross-entropy loss functions across different numbers of examples (0 to 8) with $d_E = 100$, averaged over 10 repetitions.}
    \begin{tabular}{c p{0.7cm} p{0.7cm} p{0.7cm} p{0.7cm} p{0.7cm} c p{0.7cm} p{0.7cm} p{0.7cm} p{0.7cm} p{0.7cm}}
    \toprule
    \multicolumn{1}{c}{} & \multicolumn{5}{c}{Squared} & \multicolumn{1}{c}{} & \multicolumn{5}{c}{Cross-entropy} \\
    \cmidrule(lr){2-6} \cmidrule(lr){8-12}
    $(p_0, p_1, p_2)$ & 0 & 2 & 4 & 6 & 8 & & 0 & 2 & 4 & 6 & 8 \\
    \midrule
    $(0, 1, 0)$ & 1 & 0 & 0 & 0 & 0 & & 0 & 0 & 0 & 0 & 0 \\
    $(0, 0, 1)$ & 1 & 0.97 & 0 & 0 & 0 & & 1 & 0 & 0 & 0 & 0 \\
    $(1/2, 1/2, 0)$ & 1 & 1 & 1 & 0.53 & 0 & & 1 & 0 & 0 & 0 & 0 \\
    $(1/2, 0, 1/2)$ & 1 & 1 & 1 & 1 & 1 & & 1 & 1 & 1 & 1 & 1 \\
    $(0, 1/2, 1/2)$ & 1 & 1 & 0.76 & 0 & 0 & & 1 & 1 & 0 & 0 & 0 \\
    $(1/3, 1/3, 1/3)$ & 1 & 1 & 1 & 1 & 1 & & 1 & 1 & 1 & 0.18 & 0 \\
    \bottomrule
    \end{tabular}
    \label{tab:corr-sq-ce}
\end{table}

From Tables \ref{tab:clean-sq-ce} and \ref{tab:corr-sq-ce}, we observe that 
ICL with CBOW on single-relationship tasks performs better with squared loss compared to cross-entropy loss and with fewer demonstration examples. Also, ICL tends to deteriorate after a certain number of in-context demonstrations. As detailed in Appendix \ref{app:proof-1}, a smaller number of examples (e.g., zero) allows the model to produce the correct output instead of one of the $r_i$'s. This is in contrast with transformer-based LLMs, which achieve better ICL performance as the number of demonstrations increases. On the other hand, ICL on dual-relationship tasks as described in Section \ref{sec:icl-connected} requires at least one demonstration example to distinguish between the two tasks.

\newpage

\section{Proof of Theorem \ref{prop:icl-cbow-2}}
\label{app:proof-thm-2}
\begin{proof}
We show that given a prompt of the form $c_{i_1} d_{i_1} \cdots c_{i_{\ell}} d_{i_{\ell}} c_{i_{\ell+1}}$ with distinct $i_k$'s, a trained CBOW model is more likely to predict $d_{i_{\ell+1}}$ than $e_{i_{\ell+1}}$. If this is established, the other part of the theorem follows analogously. We now define

\begin{itemize}
    \item $p_1 = p(c_i, d_j) = p(d_i, c_j) = p(d_i, d_j) = p(c_i, e_j) = p(e_i, c_j) = p(e_i, e_j)$ for any $i \neq j$;
    \item $p_2 = p(r_i, r_j)$ for any $i \neq j$;
    \item $p_3 = p(c_i, d_i) = p(d_i, c_i) = p(c_i, e_i) = p(e_i, c_i)$;
    \item $p_4 = p(d_i, r_j) = p(r_i, d_j) = p(e_i, r_j) = p(r_i, e_j)$ for any $i,j$;
\end{itemize}

where the equalities in the probabilities are a consequence of the data distribution. By direct calculation, we have $p_1 = \frac{1}{K(K-1)}$, $p_2 = \frac{(S-4)(S-5)}{L(L-1)}$, $p_3 = \frac{1}{K}$, and $p_4 = \frac{S-4}{KL}$, where each $p_i$ is multiplied by $S(S-1) > 0$ (without loss of generalization) to make calculations easier. Moreover, it is easy to see that $p(c_i, r_j) = p(r_i, c_j) = 2p_4$ for any $i,j$ and $p(c_i,c_j) = 2p_1$ for any $i \neq j$. Lastly, we define $a = (S-2)p_1$, $b = (S-2) p_2$, $c = (S-2) p_3$, $d = (S-2) p_4$, $e = 2(K-1)p_1 + p_3 + Lp_4$, and $f = 4Kp_4 + (L-1)p_2$.

The next step the proof is to use Lemma \ref{lem:1} in Appendix \ref{app:proof-1} to obtain the similarity matrix $A$. As previously, we denote a square matrix with $\alpha$ on the diagonal and $\beta$ off the diagonal as $X_{\alpha, \beta}$, and a matrix with all entries $\gamma$ as $Y_\gamma$. We then have
\[
B = \begin{bmatrix}
    X_{0, 2p_1} & X_{p_3, p_1} & X_{p_3, p_1} & Y_{2p_4} \\
    X_{p_3, p_1} & X_{0, p_1} & Y_0 & Y_{p_4} \\
    X_{p_3, p_1} & Y_0 & X_{0, p_1} & Y_{p_4} \\
    Y_{2p_4} & Y_{p_4} & Y_{p_4} & X_{0, p_2}
\end{bmatrix}
\]
and
\begin{equation}
    \label{mat:actual}
   (S-2)B + C = \begin{bmatrix}
    X_{2e,2a} & X_{c,a} & X_{c,a} & Y_{2d} \\
    X_{c,a} & X_{e,a} & Y_0 & Y_{d} \\
    X_{c,a} & Y_0 & X_{e,a} & Y_{d} \\
    Y_{2d} & Y_{d} & Y_{d} & X_{f,b}
\end{bmatrix}. 
\end{equation}

Moreover, its inverse can be written as 

\begin{equation}
    \label{mat:inverse}
    ((S-2)B + C)^{-1} = \begin{bmatrix}
    X_{q_2, q_1} & X_{q_3, q_1} & X_{q_3, q_1} & Y_{q_4} \\
    X_{q_3, q_1} & X_{q_5, q_6} & X_{q_7, q_8} & Y_{q_4} \\
    X_{q_3, q_1} & X_{q_7, q_8} & X_{q_5,q_6} & Y_{q_4} \\
    Y_{q_4} & Y_{q_4} & Y_{q_4} & X_{q_9,q_{10}}
\end{bmatrix},
\end{equation}

for some $q_1, q_2, \cdots, q_{10}$. Recall that we want to show that given $c_{i_1}, d_{i_1}, \cdots, c_{i_{\ell}}, d_{i_{\ell}} c_{i_{\ell+1}}$ with distinct $i_k$'s as context words, the center word is more likely to be $d_{i_{\ell+1}}$ than $e_{i_{\ell+1}}$. In other words, we need to establish that
\begin{align*}
&d_{i_{\ell + 1}}^\top c_{i_1} + d_{i_{\ell + 1}}^\top d_{i_1} + \cdots + d_{i_{\ell + 1}}^\top c_{i_\ell} + d_{i_{\ell + 1}}^\top d_{i_\ell} + d_{i_{\ell + 1}}^\top c_{i_\ell+1} \\
&> e_{i_{\ell + 1}}^\top c_{i_1} + e_{i_{\ell + 1}}^\top d_{i_1} + \cdots + e_{i_{\ell + 1}}^\top c_{i_\ell} + e_{i_{\ell + 1}}^\top d_{i_\ell} + e_{i_{\ell + 1}}^\top c_{i_\ell+1},
\end{align*}
where $\epsilon^\top \delta$ indicates the similarity between the word $\epsilon$ in the center and the word $\delta$ in the context. This similarity can be obtained from the matrix $A = B((S-2)B+C)^{-1}$. By symmetry, the inequality reduces to $d_i^\top d_j > e_i^\top d_j$ for any $i \neq j$.

By computing the matrix $A$, we have 

$$d_i^\top d_j = p_3 q_1 + p_1 q_3 + (K-2) p_1 q_1 + (K-2) p_1 q_6 + L p_4 q_4 + p_1 q_5$$ 

and 

$$e_i^\top d_j = p_3 q_1 + p_1 q_3 + (K-2) p_1 q_1 + (K-2)p_1 q_8 + p_1 q_7 + L p_4 q_4.$$

Thus, our problem again reduces to showing $(K-2) q_6 + q_5 > (K-2) q_8 + q_7$ as $p_1 = \frac{1}{K(K-1)} > 0$. Upon multiplying \eqref{mat:inverse} and \eqref{mat:actual} and equating the result with the identity matrix, we have the following equations:

\begin{align}
    a(K-1)q_1 + cq_3 + dLq_4 + eq_5 + a(K-1)q_6 &= 1  \label{eq:equation1}\\
    (c+a(K-2))q_1 + aq_3 + dLq_4 + aq_5 + (e+a(K-2))q_6 &= 0 \label{eq:equation2}\\
    a(K-1)q_1 + cq_3 + dLq_4 + eq_7 + a(K-1)q_8 &= 0 \label{eq:equation3} \\
    (c+a(K-2))q_1 + aq_3 + dLq_4 + aq_7 + (e+a(K-2))q_8 &= 0. \label{eq:equation4}
\end{align}

Comparing \eqref{eq:equation2} and \eqref{eq:equation4} yields 

$$a (((K-2)q_6 + q_5) - ((K-2)q_8 + q_7)) = e(q_8 - q_6).$$

As $a = (S-2)p_1 > 0$ and $e = 2p_1(K-1) + p_3 + p_4L > 0$, we now only need to show that $q_8 > q_6$. Comparing \eqref{eq:equation1} and \eqref{eq:equation3} as well as \eqref{eq:equation2} and \eqref{eq:equation4}, we have

\begin{align*}
    a(q_5 - q_7) &= (e+a(K-2))(q_8 - q_6) \\
    e(q_5 - q_7) &= a(K-1)(q_8 - q_6) + 1,
\end{align*}

which reduces to $(q_8 - q_6)(e^2 + ae(K-2) - a^2(K-1)) = a.$ The conclusion follows since $a > 0$ and
$$e^2 + ae(K-2) - a^2(K-1) = (e-a)(e+a(K-1)) = \left( \frac{S-1}{K} - \frac{S-2}{K(K-1)} \right)(e+a(K-1)) > 0.$$
\end{proof}

% \section{Corpus generation process for experiments in Section \ref{sec:synth-exp}}
% \label{app:corpus}

% \begin{enumerate}
%     \item Randomly select 10 countries and obtain their capital cities and IOC codes.
%     \item Generate 30 sentences containing exactly one country-capital pair (3 for each country).
%     \item Generate 30 sentences containing exactly one country-IOC pair (3 for each country).
%     \item Generate 30 sentences containing exactly one country without any pair.
%     \item Generate 60 sentences without any country, capital city, or IOC code.
%     \item Generate 810 sentences containing exactly two different country-capital pairs by concatenating sentences generated in Step 2.
%     \item Generate 810 sentences containing exactly two different country-IOC pairs by concatenating sentences generated in Step 3.
% \end{enumerate}

\newpage

\section{Proof of Theorem \ref{prop:icl-pe}}
\label{app:proof-2}
\begin{proof}
Consider the instance of predicting $a$ from $abc$, i.e., $f(\{a,b\},c)$. By the assumption on the data distribution, it is equally likely that the task is predicting $b$ from $bac$. In this case, the corresponding function is also $f(\{a,b\},c)$. Thus, the sum of the cross-entropy losses corresponding to these two tasks is lower bounded by $2 \log 2 > 0$. Also, it is easy to see that we cannot achieve perfect accuracy since the predictions for $abc$ and $bac$ must be the same.

We now show that it is possible to attain zero loss and perfect accuracy when the model includes positional embeddings, so that  
$\tilde{f}(\{ (a, 1), (b, 2)\}, (c, 3)) \neq \tilde{f}(\{ (b, 1), (a, 2)\}, (c, 3))$. As a special case, we consider a simplified version of the transformer architecture, where 

$$\tilde{f}(\{ (a, 1), (b, 2)\}, (c, 3)) = \frac{\sum_{k \in \{a,b,c\}} (x_k + p_1) \exp((x_k + p_1)^\top (x_c + p_3))}{\sum_{k \in \{a,b,c\}} \exp((x_k + p_1)^\top (x_c + p_3))}.$$

and 

$$p(d \mid abc) \propto \exp\left( x_d^\top  \tilde{f}(\{ (a, 1), (b, 2)\}, (c, 3)) \right). $$

for any token $d$. Here, $x_i$ and $p_j$ represent the embedding of token $i$ and position $j$, respectively.

Let $p_1^\top p_3 = p$, $p_2^\top p_3 = q$, $p_3^\top p_3 = r$, $x_i^\top x_i = s$, $x_i^\top x_j = t$ for any $i \neq j$, $p_1^\top x_i = u$ for any $i$, $p_2^\top x_i = v$ for any $i$, and $p_3^\top x_i = w$ for any $i$. Note that this holds due to the assumed data generating process. We consider the following construction: $p_1 = b \mathds{1}_{|V|}$, $p_2 = p_3 = \mathds{1}_{|V|}$, and $x_i = a e_i$, where $e_i$ is a zero vector with $1$ on the $i$-th entry. This implies $p = b|V|$, $q = r = |V|$, $s = a^2$, $t = 0$, $u = ab$, and $v = w = a$. 

By direct calculation, the cross-entropy loss of predicting $a$ from $abc$ is given by

$$-\log\left( \frac{\exp(\alpha_1 a^2)}{\exp(\alpha_1 a^2) + \exp(\alpha_2 a^2) + \exp(\alpha_3 a^2) + |V| - 3}\right),$$

where 

$$\alpha_1 = \frac{\exp(ab + b|V|)}{\exp(ab + b|V|) + \exp(a+|V|) + \exp(a^2+a+|V|)},$$

$$\alpha_2 = \frac{\exp(a + |V|)}{\exp(ab + b|V|) + \exp(a+|V|) + \exp(a^2+a+|V|)},$$

$$\alpha_3 = \frac{\exp(a^2 + a + |V|)}{\exp(ab + b|V|) + \exp(a+|V|) + \exp(a^2+a+|V|)}.$$

Letting $b = a^2$ and $a \rightarrow \infty$, it is easy to see that we can bring the cross-entropy loss arbitrarily close to zero. Consequently, we also have a perfect prediction accuracy.
\end{proof}

\section{Proof of Proposition \ref{lem:pe}}
\label{app:proof-pe}

\begin{proof}
    The intermediate representation of the first layer is given by $f_1(\{x_{i_1}\})$, $f_2(\{x_{i_1}\},x_{i_2}) \}$, $f_3(\{x_{i_1},x_{i_2}\},x_{i_3})$, and $f_4(\{x_{i_1},x_{i_2}, x_{i_3}\},x_{i_1})$, for some functions $f_1, f_2, f_3$, and $f_4$. To predict the last $x_{i_1}$, we use the third coordinate of the second layer representation, which is given by $t(x_{i_1} x_{i_2} x_{i_3}) := g_3\left( \{ f_1(\{x_{i_1}\}), f_2(\{x_{i_1}\},x_{i_2}) \}, f_3(\{x_{i_1},x_{i_2}\},x_{i_3}) \right)$, for some function $g_3$. It is easy to see that in general, $t(x_{i_1} x_{i_2} x_{i_3}) \neq t(x_{i_2} x_{i_1} x_{i_3})$.
\end{proof}

\section{Proof of Theorem \ref{prop:icl-one-pattern}}
\label{app:proof-icl-one-pattern}

\begin{proof}

In the one-noisy scenario, each sentence takes one of the following forms: \textit{nabacdc}, \textit{anbacdc}, \textit{abnacdc}, \textit{abancdc}, \textit{abacndc}, and \textit{abacdnc}, where \textit{n} $\in \mathcal{N}$. In order to achieve the minimum possible theoretical loss, we minimize each loss term separately. Concretely, the minimum loss of predicting the sixth token given the first five tokens is attained by the following rule:

\begin{itemize}
    \item When the first five tokens do not contain any nuisance token, output a uniform probability vector over $\mathcal{N}$.
    \item Otherwise, output the conditional probability of $c[2]$ given $x$, where $(x, c[2]) \in S_1$. Here, $x$ represents the last non-nuisance token.
\end{itemize}

Under this rule, the predicted output for any in-context example \textit{\underline{abacd}} is never $c$, since $c \notin \mathcal{N}$. In the block-noisy scenario, each sentence takes one of the following forms: $n_1 n_2 n_3 aba cdc$, $aba n_1 n_2 n_3 cdc$, and $aba cdc n_1 n_2 n_3$, where $n_1, n_2, n_3 \in \mathcal{N}$. The minimum loss of predicting the ninth token given the first eight tokens is attained by the following rule:

\begin{itemize}
    \item When the seventh token is not a nuisance token, output the seventh token with probability one.
    \item When the seventh token is a nuisance token, output a uniform probability vector over $\mathcal{N}$.
\end{itemize}

Under this rule, the predicted output for any in-context example \textit{\underline{abacdcef}} is \textit{e}, resulting in perfect ICL accuracy.

\end{proof}
\newpage
\section{Proof of Theorem \ref{prop:failed-1}}
\label{app:proof-failed-1}

\begin{proof}
Recall that each training sentence is of the form $x_{11}x_{12}x_{11}x_{21}x_{22}x_{21} \cdots x_{N1}x_{N2}x_{N1}$. Note that we can decompose the total loss $\mathcal{L}$ into $\mathcal{L}_1 + \mathcal{L}_2 + \cdots + \mathcal{L}_{3N}$, where $\mathcal{L}_g$ denotes the loss of predicting the $g$-th token given all the other previous tokens. As the $x_{i1}x_{i2}x_{i1}$ blocks are generated independently, the optimal loss should satisfy $\mathcal{L}_1 = \mathcal{L}_4 = \cdots = \mathcal{L}_{[3N-2]} = \mathcal{L}_{[1]}$, $\mathcal{L}_2 = \mathcal{L}_5 = \cdots = \mathcal{L}_{[3N-1]} = \mathcal{L}_{[2]}$, and $\mathcal{L}_3 = \mathcal{L}_6 = \cdots = \mathcal{L}_{[3N]} = \mathcal{L}_{[3]}$. Therefore, it is sufficient to minimize $\mathcal{L}_{[1]} + \mathcal{L}_{[2]} + \mathcal{L}_{[3]}$. 

In order to achieve the minimum possible theoretical loss, we need to minimize $\mathcal{L}_{[1]}$, $\mathcal{L}_{[2]}$, and $\mathcal{L}_{[3]}$ separately. It is easy to see that $\mathcal{L}_{[1]}$ is minimized by outputting the marginal probability of $c[1]$, where $c \in S_1$. Similarly, $\mathcal{L}_{[2]}$ is minimized by outputting the conditional probability of $c[2]$ given $x_{i_1}$, where $(x_{i_1}, c[2]) \in S_1$. On the other hand, it is possible to achieve an $\mathcal{L}_{[3]}$ value of zero by outputting $x_{i_1}$ with probability one.  

Now, given an ICL prompt $\underline{x_{11}x_{12}x_{12}x_{21}x_{22}x_{22}} \cdots \underline{x_{\ell1}x_{\ell2}}$ where $\ell \leq N$, the trained model should predict $x_{\ell 1}$ with probability one since $\{c[1] \mid c \in S_2\} = \mathcal{V}$ and our ICL prompt corresponds to $\mathcal{L}_{[3]}$. This completes the proof.
\end{proof}

\section{Proof of Theorem \ref{prop:failed-2}}
\label{app:proof-failed-2}

\begin{proof}
We proceed similarly as the proof of Theorem \ref{prop:failed-1}. Concretely, we separately minimize $\mathcal{L}_g$ for $g \in [2k+2]$, where $\mathcal{L}_g$ denotes the loss of predicting the $g$-th token given all the other previous tokens. It is easy to see that $\mathcal{L}_1$ is minimized by outputting a uniform probability vector over $a_{1:[I]}$, whereas $\mathcal{L}_h$ (for any $2 \leq h \leq 2k+1$) is minimized by outputting a uniform probability vector over $\mathcal{V}$. Moreover, it is possible to achieve an $\mathcal{L}_{2k+2}$ value of zero by outputting $b_i$ with probability one.

From here, given an ICL prompt of the form $a_{i_1}b_{i_1}a_{i_2}b_{i_2} \cdots a_{i_\ell}$, the trained model should predict a uniform probability vector over $\mathcal{V}$ if $\ell \leq k$, and $b_{i_1}$ if $\ell = k+1$. In all cases, the model does not predict $b_{i_\ell}$, completing the proof.
\end{proof}

\newpage

% \section{Distinguishing two different patterns requires more than one layers}
% \label{app:one-vs-five}

% \begin{figure}[h]
%     \centering
%     \begin{subfigure}[t]{0.49\textwidth} % Adjust width as needed
%         \centering
%         \includegraphics[width=\textwidth]{one-layer.png}
%     \end{subfigure}
%     \hfill % This will insert a space between the two figures
%     \begin{subfigure}[t]{0.49\textwidth} % Adjust width as needed
%         \centering
%         \includegraphics[width=\textwidth]{five-layer.png}
%     \end{subfigure}
%     \caption{One-layer models fail to differentiate the two patterns in Section \ref{sec:two-patterns}, as evidenced by the accuracy trajectory graph on the left. On the other hand, five-layer models are capable of doing so.}
%     \label{fig:one-vs-five}
% \end{figure}

\section{Details of experiments and data sets}
\label{app:exp-details}

\subsection{Architecture and implementation}
All experiments {utilize the Keras package in Python, employing} the Adam optimizer \citep{kingma2015adam} with a learning rate of 0.01. Early stopping is applied based on validation loss with a patience threshold of 5, utilizing a randomly selected subset representing 50\% of the original data set. {Each transformer layer uses two heads, as we empirically demonstrated that increasing the number of heads does not impact performance in our experiments. Each layer consists of the following components (in order): (1) Keras' multi-head causal self-attention block, with key\_dim $=$ value\_dim $=$ embed\_dim$/2$; (2) Skip connection and layer normalization; (3) One hidden layer feed-forward network using the ReLU activation with dimension $= 2 \times$ embed\_dim; and (4) Skip connection and layer normalization.} 

\subsection{Source and details of data sets}
The \textit{world\_population.csv} data set, used for the experiments in Section \ref{sec:co-occ}, is obtained from \href{https://www.kaggle.com/datasets/iamsouravbanerjee/world-population-dataset}{Kaggle}. According to the author, this data set is created from \href{https://worldpopulationreview.com/}{World Population Review}. 

The \textit{us-state-capitals.csv} data set, used for the experiments in the beginning of Section \ref{sec:co-occ}, is obtained from \href{https://github.com/jasperdebie/VisInfo/blob/master/us-state-capitals.csv}{this Github repository}. Its source is unclear.

The \textit{uscities.csv} data set, used for the experiments in the beginning of Section \ref{sec:co-occ}, is obtained from \href{https://simplemaps.com/data/us-cities}{Simple Maps}, with a CC 4.0 license.

\subsection{Details of synthetic data used in experiments}

{
Below we provide additional details regarding the synthetic data used in our experiments.}
\begin{enumerate}
    \item {For experiments in Table \ref{tab:one-rel}, the training data consists of 50,000 sentences. In the clean version, sentences are generated uniformly as described in Section \ref{sec:icl-single}. In the corrupted version, sentences are generated in a similar manner, but each $(c_i, d_i)$ pair is replaced by $(c_i, r_j)$ or $(d_i, r_j)$ with a probability of $1/4$ each. Test sentences are generated according to the setup in Theorem \ref{prop:icl-cbow}. Some examples are as follows:}
    \begin{itemize} 
    \item Clean
    \begin{itemize}
        \item Training: $c_1 d_1 r_1 r_2 r_3 r_4 r_5 r_6$ or $r_1 r_2 r_3 r_4 r_5 r_6 r_7 r_8$
        \item Prompt: $c_1 d_1 c_2 d_2 c_3 d_3 c_4 ?$
    \end{itemize}
    \item Corrupted
    \begin{itemize}
        \item Training: $c_1 r_1 r_2 r_3 r_4 r_5 r_6 r_7$ or $c_1 d_1 c_2 r_1 r_2 r_3 r_4 r_5$
        \item Prompt: $c_1 d_1 c_2 d_2 c_3 d_3 c_4 ?$
    \end{itemize}

    \end{itemize}
    \item {For experiments in Table \ref{tab:two-con}, the training data consists of 50,000 sentences. In the clean version, sentences are generated uniformly as described in Section \ref{sec:icl-connected}. In the imbalanced and extreme versions, the 60 other words are divided into three categories: 20 for $cd$ sentences ($rcd_\cdot$), 20 for $ce$ sentences ($rce_\cdot$), and 20 for both types ($r_\cdot$). In the imbalanced version, $cd$ ($ce$) sentences are 4 times more likely to sample a $cd$ ($ce$) word than a $ce$ ($cd$) word. In the extreme version, $cd$ ($ce$) sentences cannot contain any $ce$ ($cd$) words. Test sentences are generated according to the setup in Theorem \ref{prop:icl-cbow-2}. Some examples are as follows:}
    \begin{itemize}
        \item Clean examples
        \begin{itemize}
            \item Training: $c_1 d_1 r_1 r_2 r_3 r_4 r_5 r_6$ or $c_1 e_1 r_1 r_2 r_3 r_4 r_5 r_6$
            \item Prompt: $c_1 d_1 c_2 d_2 c_3 d_3 c_4 ?$ or $c_1 e_1 c_2 e_2 c_3 e_3 c_4 ?$
        \end{itemize}
        \item Imbalance examples
        \begin{itemize}
            \item Training: $c_1 d_1 rcd_1 rcd_2 rcd_3 rce_4 r_5 r_6$ or $c_1 e_1 rcd_1 r_2 rce_3 rce_4 rce_5 r_6$
            \item Prompt: $c_1 d_1 c_2 d_2 c_3 d_3 c_4 ?$ or $c_1 e_1 c_2 e_2 c_3 e_3 c_4 ?$
        \end{itemize}
        \item Extreme examples
        \begin{itemize}
            \item Training: $c_1 d_1 rcd_1 rcd_2 rcd_3 r_4 r_5 r_6$ or $c_1 e_1 r_1 r_2 r_3 rce_4 rce_5 r_6$
            \item Prompt: $c_1 d_1 c_2 d_2 c_3 d_3 c_4 ?$ or $c_1 e_1 c_2 e_2 c_3 e_3 c_4 ?$
        \end{itemize}
    \end{itemize}
    \item Experiments in Table \ref{tab:two-discon} follow the setup of experiments in Table \ref{tab:two-con}, except that the pairs are now of the form $(c_i, d_i)$ and $(e_i, f_i)$ instead of $(c_i, d_i)$ and $(c_i, e_i)$.
    \item For experiments in Section \ref{sec:synth-exp}, the corpus generation process is as follows:
    \begin{itemize}
        \item Randomly select 10 countries and obtain their capital cities and IOC codes.
        \item Generate 30 sentences containing exactly one country-capital pair (3 for each country). \textit{Example: Paramaribo is the vibrant heart of Suriname.}
        \item Generate 30 sentences containing exactly one country-IOC pair (3 for each country). \\ \textit{Example: Gabon (GAB) protects its diverse rainforests and wildlife.}
        \item Generate 30 sentences containing exactly one country without any pair. \\
        \textit{Example: The banking sector is central to Liechtenstein's prosperity.}
        \item Generate 60 sentences without any country, capital city, or IOC code. \\
        \textit{Example: Every country has its unique cultural identity and heritage.}
        \item Generate 810 sentences containing exactly two different country-capital pairs by concatenating sentences generated in Step 2. \\
        \textit{Example: The city of Dushanbe reflects Tajikistan's vibrant spirit. Roseau is the cultural tapestry of Dominica.}
        \item Generate 810 sentences containing exactly two different country-IOC pairs by concatenating sentences generated in Step 3. \\
        \textit{Example: Mayotte (MAY) features lush landscapes and peaks. Turkmenistan (TKM) features the fiery Darvaza Crater.}
    \end{itemize}

    The ICL prompts follow the form used in the country-capital city and US state-capital city experiments in the beginning of Section \ref{sec:co-occ}, with 1 to 5 in-context examples.

   \item For experiments in Table \ref{tab:icl-pos}, the training and test data consist of all sentences in the form $abca$, where $a$, $b$, and $c$ are distinct. Each test sentence is different from any training sentence. In the first scenario (both), the first tokens of the training sentences cover the entire vocabulary. In the second scenario (either), each token can be the first token in either the training or test data, but not both.

   \item For experiments in Table \ref{tab:icl-one-pattern}, the training data consists of 50,000 sentences generated uniformly as detailed in Section \ref{sec:one-pattern}. The ICL prompt formats are also described in Section \ref{sec:one-pattern}.

    \item For experiments in Table \ref{tab:icl-two-patterns}, the training data consists of 50,000 sentences. In the clean scenario, the training data are of the form $abcadefd$ and $abcbdefe$, with ICL prompts as $\overline{abcadef} ?$ and $\overline{abcbdef} ?$. In the block-noisy scenario, the training data include sequences like $n_1 n_2 n_3 n_4 abcadefd$ and $abcb n_1 n_2 n_3 n_4 defe$, with ICL prompts as $\overline{abcadefdghi} ?$ and $\overline{abcbdefeghi} ?$.

    \item For experiments in Table \ref{tab:icl-fail}, the training data consists of 50,000 sentences generated uniformly according to the processes in Sections \ref{sec:failed-1} and \ref{sec:failed-2}. The ICL prompt formats are also described in the same subsections.

    \item For experiments in Section \ref{sec:failed-synth}, the corpus generation process is as follows:
    \begin{itemize}
    \item Randomly select 10 countries and obtain their capital cities and IOC codes.
    \item Generate 130 sentences containing exactly one country-capital pair (13 for each country). \\ \textit{Example: Paramaribo stands as capital of Suriname.}
    \item Generate 30 sentences containing exactly one country without any pair. \\
    \textit{Example: The banking sector is central to Liechtenstein's prosperity.}
    \item Generate 60 sentences without any country, capital city, or IOC code. \\
    \textit{Example: Every country has its unique cultural identity and heritage.}
    \item Generate 1,000 sentences containing exactly two different country-capital pairs by concatenating sentences generated in Step 2. \\
    \textit{Example: Brazil functions as heart of Brasilia. Turkmenistan operates as center for Ashgabat.}
    \end{itemize}

\end{enumerate}

\end{document}